\documentclass[11pt]{article}
\pdfoutput=1
\usepackage{enumerate}
\usepackage{pdfsync}
\usepackage[OT1]{fontenc}

\usepackage[usenames]{color}
\usepackage{smile}
\usepackage[colorlinks,
            linkcolor=red,
            anchorcolor=blue,
            citecolor=blue
            ]{hyperref}
\usepackage{fullpage}
\usepackage{hyperref}
\usepackage[protrusion=true,expansion=true]{microtype}
\usepackage{pbox}
\usepackage{setspace}
\usepackage{tabularx}
\usepackage{float}
\usepackage{wrapfig,lipsum}
\usepackage{enumitem}
\usepackage{microtype}
\usepackage{graphicx}
\usepackage{subfigure}
\usepackage{enumerate}
\usepackage{enumitem}
\usepackage{pgfplots}
\usetikzlibrary{arrows,shapes,snakes,automata,backgrounds,petri}
\usepackage{booktabs} 

\newcommand{\la}{\langle}
\newcommand{\ra}{\rangle}
\newcommand{\qvalue}{Q}
\newcommand{\vvalue}{V}

\newcommand{\reward}{r}

\newcommand{\valueite}{\text{EVI}}

\def \algname {\text{UCLK}}
\newcommand{\state}{x}
\usepackage{enumitem}
\newcommand{\alglinelabel}{%
  \addtocounter{ALC@line}{-1}
  \refstepcounter{ALC@line}
  \label
}

\makeatletter
\newcommand*{\rom}[1]{\expandafter\@slowromancap\romannumeral #1@}
\makeatother
\title{\huge Provably Efficient Reinforcement Learning for Discounted MDPs with Feature Mapping}

\author
{
	Dongruo Zhou\thanks{Department of Computer Science, University of California, Los Angeles, CA 90095, USA; e-mail: {\tt drzhou@cs.ucla.edu}} 
	~~~and~~~
	Jiafan He\thanks{Department of Computer Science, University of California, Los Angeles, CA 90095, USA; e-mail: {\tt jiafanhe19@ucla.edu}} 
	~~~and~~~
	Quanquan Gu\thanks{Department of Computer Science, University of California, Los Angeles, CA 90095, USA; e-mail: {\tt qgu@cs.ucla.edu}}
}

\begin{document}
\date{}
\maketitle

\begin{abstract}%
 Modern tasks in reinforcement learning have large state and action spaces. To deal with them efficiently, one often uses predefined feature mapping to represent states and actions in a low-dimensional space. In this paper, we study reinforcement learning for discounted Markov Decision Processes (MDPs), where the transition kernel can be parameterized as a linear function of certain feature mapping. 
 We propose a novel algorithm that makes use of the feature mapping and obtains a $\tilde O(d\sqrt{T}/(1-\gamma)^2)$ regret, where $d$ is the dimension of the feature space, $T$ is the time horizon and $\gamma$ is the discount factor of the MDP. To the best of our knowledge, this is the first polynomial regret bound without accessing the generative model or making strong assumptions such as ergodicity of the MDP. By constructing a special class of MDPs, we also show that for any algorithms, the regret is lower bounded by  $\Omega(d\sqrt{T}/(1-\gamma)^{1.5})$. Our upper and lower bound results together suggest that the proposed reinforcement learning algorithm is near-optimal up to a $(1-\gamma)^{-0.5}$ factor.
\end{abstract}

\section{Introduction}
Designing efficient algorithms that learn and plan in sequential decision-making tasks with large state and action spaces has become the central goal of modern reinforcement learning (RL) in recent years. Due to numerous possible states and actions, traditional tabular reinforcement learning methods \citep{watkins1989learning, jaksch2010near, azar2017minimax} which directly access each state-action pair are computationally intractable. A common method to design reinforcement learning algorithms for large-scale state and action spaces is to make use of feature mappings such as linear functions or neural networks to map states and actions to a low-dimensional space and solve the decision-making problem in the feature space. Despite the empirical success of feature mapping based reinforcement learning methods  \citep{singh1995reinforcement,bertsekas2018feature}, the theoretical understanding and the fundamental limits of these methods remain largely understudied. 

In this paper, we aim to develop provable reinforcement learning algorithms with feature mapping for discounted Markov Decision Processes (MDPs). Discounted MDP is one of the most widely used models to formulate the modern reinforcement learning tasks such as Atari games \citep{mnih2015human} and deep recommendation system \citep{zheng2018drn}. With feature mapping, a series of recent work \citep{yang2019sample,lattimore2019learning, bhandari2018finite,zou2019finite} have proposed provably efficient algorithms along with theoretical guarantees. However, these existing results either rely on a special oracle called \emph{generative model} \citep{kakade2003sample} that allows an algorithm to query any possible state-action pairs and return both the reward and the next state \citep{yang2019sample, lattimore2019learning}, or needs strong assumptions such as uniform ergodicity \citep{bhandari2018finite, zou2019finite} on the underlying MDP. 
A natural question arises:

\begin{center}
   \emph{Can we design provably efficient RL algorithms with feature mapping for discounted MDPs under mild assumptions?} 
\end{center}

We answer this question affirmatively. To be more specific, we consider a special class of discounted MDPs called \emph{linear kernel MDP}, where the transition probability kernel can be represented as a linear function of a predefined $d$-dimensional feature mapping. A similar model has been studied in earlier work \citet{jia2020model, ayoub2020model} for finite horizon episodic MDPs, where the authors call it \emph{linear mixture model}. Linear kernel MDP is a rich MDP class, which covers many classes of MDPs proposed in previous work \citep{yang2019reinforcement, modi2019sample} as special cases. 
We propose a novel provably efficient algorithm namely Upper-Confidence Linear Kernel reinforcement learning ($\algname$) to solve this MDP. We prove both upper and lower regret bounds and show that our algorithm is near-optimal under the linear kernel MDP setting. 

Our contributions are summarized as follows.
\begin{itemize}
    \item We propose a novel algorithm $\algname$ to learn the optimal value function with the help of predefined feature mapping. We show that the regret (See Definition~\ref{def:regret}) for $\algname$ to learn the optimal value function is $\tilde O(d\sqrt{T}/(1-\gamma)^2)$. It is worth noting that the regret is independent of the cardinality of the state and action spaces, which suggests that $\algname$ is efficient for large-scale RL problems. To the best of our knowledge, this is the first feature-based reinforcement learning algorithm that attains a polynomial regret bound for discounted MDPs without accessing the generative model or making strong assumptions on MDPs such as ergodicity\footnote{Without a generative model (simulator) or further assumptions on MDP, some states may never be visited starting from certain initial states, which makes it impossible to find a near-optimal policy on them. Therefore, it is not meaningful to consider the sample complexity of UCLK to find an $\epsilon$-optimal policy.}.
    \item We also show that for any reinforcement learning algorithms, the regret to learn the optimal value function in linear kernel MDP is at least $\Omega(d\sqrt{T}/(1-\gamma)^{1.5})$. This lower bound result suggests that $\algname$ is optimal concerning feature mapping dimension $d$ and time horizon $T$, and it is near-optimal concerning the discount factor up to $(1-\gamma)^{-0.5}$. Our proof is based on a specially constructed linear kernel MDP, which could be of independent interest.
\end{itemize}

After we posted the first version of this paper online, we were informed that 
the linear kernel MDP setting is the same as the so-called \emph{parameterized transition mode} or \emph{linear mixture model} in earlier work \citep{jia2020model, ayoub2020model}.

The remainder of this paper is organized as follows. In Section \ref{section 2}, we review the related work in the literature. We introduce preliminaries in in Section \ref{section 3}, and our algorithm in Section \ref{section 4}. In Section~\ref{section 5}, we present our main theoretical results including both upper and lower regret bounds, followed by a proof sketch of the main theory in Section~\ref{sec:sketch}. Finally, we conclude this paper in Section~\ref{section 8}. The detailed proofs are deferred to the supplementary material.

\noindent\textbf{Notation} 
We use lower case letters to denote scalars, and use lower and upper case bold face letters to denote vectors and matrices respectively. Let $\ind(\cdot)$ denote the indicator function. For a vector $\xb\in \RR^d$ and matrix $\bSigma\in \RR^{d\times d}$, we denote by $\|\xb\|_2$ the Euclidean norm and denote by $\|\xb\|_{\bSigma}=\sqrt{\xb^\top\bSigma\xb}$. For two sequences $\{a_n\}$ and $\{b_n\}$, we write $a_n=O(b_n)$ if there exists an absolute constant $C$ such that $a_n\leq Cb_n$, and we write $a_n=\Omega(b_n)$ if there exists an absolute constant $C$ such that $a_n\geq Cb_n$. We use $\tilde O(\cdot)$ to further hide the logarithmic factors. 

\section{Related Work}\label{section 2}

\noindent \textbf{Finite-horizon MDPs with feature mappings.} There is a series of work focusing on solving finite-horizon MDP using RL with function approximation \citep{jin2019provably, yang2019reinforcement, wang2019optimism, modi2019sample, jiang2017contextual, zanette2020learning, du2019good}. For instance, \citet{jin2019provably} assumed the underlying transition kernel and reward function are linear functions of a $d$-dimensional feature mapping and proposed an RL algorithm with $\tilde O(\sqrt{d^3H^3T})$ regret, where $H$ is the length of an episode. \citet{yang2019reinforcement} assumed the probability transition kernel is bilinear in two feature mappings in dimension $d$ and $d'$, and proposed an algorithm with $\tilde O(dH^2\sqrt{T})$ regret. \citet{wang2019optimism} assumed the Bellman backup of any value function is a generalized linear function of certain feature mapping and proposed an algorithm with a regret guarantee. \citet{modi2019sample} assumed the underlying MDP can be represented as a linear combination of several base models and proposed an RL algorithm to solve it with a provable guarantee. \citet{jiang2017contextual} assumed the underlying MDP is of low inherent Bellman error and proposed an algorithm with polynomial PAC bounds.  \citet{jia2020model} studied the linear mixture model and proposed a UCRL-VTR algorithm for finite-horizon MDPs which
achieves a $\tilde O(d\sqrt{H^3T})$ regret, where $H$ is the episode length. \citet{ayoub2020model} considered the same model but with general function approximation, and proved a regret bound depending on Eluder dimension \citep{russo2013eluder}. \citet{jia2020model, ayoub2020model} also proved a lower bound of regret by considering the hard tabular MDP firstly proposed in \citet{jaksch2010near}. \citet{zanette2020learning} studied a similar MDP as \citet{jin2019provably} and proposed an algorithm with tighter regret bound. \citet{du2019good} suggested that the sample complexity to learn the optimal policy can be exponential if the approximation error to the value function is moderate. More discussions and insights regarding these negative results can be found in \citet{van2019comments, lattimore2019learning}. 


\noindent \textbf{Discounted MDPs with a generative model.} For tabular discounted MDPs, many work focuses on RL with the help of a generative model (or called a simulator) \citep{kakade2003sample}. To learn the optimal value function, \citet{azar2013minimax} proposed Empirical QVI, which learns an $\epsilon$-suboptimal value function with $\tilde O(|\cS||\cA|/((1-\gamma)^3\epsilon^2))$ optimal sample complexity.
To learn the optimal policy, 
\citet{kearns1999finite} proposed Phased Q-Learning which learns an $\epsilon$-suboptimal policy with $\tilde O(|\cS||\cA|/((1-\gamma)^7\epsilon^2))$ sample complexity.
\citet{sidford2018variance} proposed
a Sublinear Randomized Value Iteration algorithm which achieves a 
$\tilde O(|\cS||\cA|/((1-\gamma)^4\epsilon^2))$ sample complexity. \citet{sidford2018near} further proposed Variance-Reduced QVI algorithm which achieves the optimal $\tilde O(|\cS||\cA|/((1-\gamma)^3\epsilon^2))$ sample complexity. For discounted MDPs with function approximation, \citet{yang2019sample} assumed the probability transition kernel can be parameterized by a $d$-dimensional feature mapping and proposed a Phased Parametric Q-Learning algorithm which learns an $\epsilon$-suboptimal policy with the optimal $\tilde O(d/((1-\gamma)^3\epsilon^2))$ sample complexity. \citet{lattimore2019learning} considered a similar setting to \citet{yang2019sample} and proposed a Phased Elimination algorithm with $\tilde O(d/((1-\gamma)^4\epsilon^2))$ sample complexity. 

\noindent \textbf{Discounted MDPs without a generative model.} Another line of work aims at learning the discounted MDP without accessing the generative model. \citet{szita2010model} proposed an MoRmax algorithm which achieves $\tilde O(|\cS||\cA|/((1-\gamma)^6\epsilon^2))$ sample complexity of exploration. \citet{lattimore2012pac} proposed UCRL algorithm which achieves $\tilde O(|\cS|^2|\cA|/((1-\gamma)^3\epsilon^2))$ sample complexity of exploration. \citet{strehl2006pac} proposed delay-Q-learning with  $\tilde O(|\cS||\cA|/((1-\gamma)^8\epsilon^4))$ sample complexity of exploration. \citet{dong2019q} proposed Infinite Q-learning with UCB which achieves $\tilde O(|\cS||\cA|/((1-\gamma)^7\epsilon^2))$ sample complexity of exploration. \citet{liu2020regret} proposed the regret definition for discounted MDPs and presented Double Q-Learning to achieve $\tilde O(\sqrt{|\cS||\cA|T}/{(1-\gamma)^{2.5}})$ regret. Our work falls into this category, and also uses regret to characterize the performance of RL.


\section{Preliminaires}\label{section 3}
We consider infinite-horizon discounted Markov Decision Processes (MDPs), which is denoted by a tuple $M(\cS, \cA, \gamma, \reward, \PP)$. Here $\cS$ is a countable state space (may be infinite), $\cA$ is the action space, $\gamma: 0 \leq \gamma <1$ is the discount factor, $\reward: \cS \times \cA \rightarrow [0,1]$ is the reward function. For simplicity, we assume the reward function $\reward$ is \emph{deterministic} and \emph{known}. $\PP(s'|s,a) $ is the transition probability function which denotes the probability for state $s$ to transfer to state $s'$ given action $a$. A (nonstationary) policy $\pi$ is a collection of policies $\pi_t$, where each $\pi_t: \{\cS \times \cA\}^{t-1}\times \cS\rightarrow \cA$ maps history $s_1, a_1, \dots, s_{t-1}, a_{t-1}, s_t$ to an action $a$. Let $\{s_t, a_t\}_{t=1}^\infty$ are states and actions deduced by $\PP$ and $\pi$. We denote the action-value function $\qvalue^\pi_t(s,a)$ and value function $\vvalue^\pi_t(s,a)$ as follows
\begin{align}
&\qvalue^\pi_t(s,a) = \EE\bigg[\sum_{i = 0}^\infty \gamma^{i}\reward(s_{t+i}, a_{t+i})\bigg|s_1,\dots, s_t = s, a_t = a\bigg],\notag\\
&\vvalue^\pi_t(s,a) = \EE\bigg[\sum_{i = 0}^\infty \gamma^{i}\reward(s_{t+i}, a_{t+i})\bigg|s_1,\dots, s_t = s\bigg].\notag
\end{align}
We define the optimal value function $V^*$ and the optimal action-value function $\qvalue^*$ as $V^*(s) = \sup_{\pi}\vvalue^{\pi_1}(s)$ and $\qvalue^*(s,a) = \sup_{\pi}\qvalue^{\pi_1}(s,a)$.
For simplicity, for any function $\vvalue: \cS \rightarrow \RR$, we denote $[\PP \vvalue](s,a)=\EE_{s' \sim \PP(\cdot|s,a)}\vvalue(s')$. Therefore we have the following Bellman equation, as well as the Bellman optimality equation:
\begin{align}
    &\qvalue^\pi_t(s_t,a_t) = \reward(s_t,a_t) +\gamma [\PP\vvalue^\pi_{t+1}](s_t,a_t),\ \qvalue^*(s_t,a_t) = \reward(s_t,a_t) + \gamma[\PP\vvalue^*](s_t,a_t).\notag
\end{align}

In this work, we consider a special class of MDPs called \emph{linear kernel MDPs}, where the transition probability function can be represented as a linear function of a given feature mapping $\bphi: \cS \times \cA \times \cS \rightarrow \RR^d$. It is worth noting that this is essentially the same MDP class as \emph{linear mixture model} considered in \citet{jia2020model, ayoub2020model}. Formally speaking, we have the following assumption for a linear kernel MDP.
\begin{definition}\label{assumption-linear}
$M(\cS, \cA, \gamma, \reward, \PP)$ is called a linear kernel MDP if there exist a \emph{known} feature mapping $\bphi(s'|s,a): \cS \times \cA \times \cS \rightarrow \RR^d$ and an \emph{unknown} vector $\btheta \in \RR^d$ with $\|\btheta\|_2 \leq \sqrt{d}$, such that 
\begin{itemize}[leftmargin = *]
    \item For any state-action-state triplet $(s,a,s') \in \cS \times \cA \times \cS$, we have $\PP(s'|s,a) = \la \bphi(s'|s,a), \btheta\ra$;
    \item For any bounded function $\vvalue: \cS \rightarrow [0,R]$ and any tuple $(s,a)\in \cS \times \cA$, we have $\|\bphi_{{\vvalue}}(s,a)\|_2 \leq \sqrt{d}R$, where $\bphi_{{\vvalue}}(s,a) = \sum_{s'}\bphi(s'|s,a)\vvalue(s') \in \RR^d$.\label{def:bphi}
\end{itemize}
We denote the linear kernel MDP by $M_{\btheta}$ for simplicity.
\end{definition}


As we will show in the following examples, linear kernel MDPs cover several MDPs studied in previous work as special cases. 
\begin{example}[Tabular MDPs]\label{exp: tabular}
For an MDP $M(\cS, \cA, \gamma, \reward, \PP)$ with $|\cS|, |\cA| \leq \infty$, the transition probability function can be parameterized by $|\cS|^2|\cA|$ \emph{unknown} parameters. 
The tabular MDP is a special case of linear kernel MDPs with the following feature mapping and parameter vector: $d = |\cS|^2|\cA|,\ \bphi(s'|s,a) = \eb_{(s,a,s')} \in \RR^d,\ \btheta = [\PP(s'|s,a)]\in\RR^d$,
where $\eb_{(s,a,s')}$ denotes the corresponding natural basis in the $d$-dimensional Euclidean space.
\end{example}

\begin{example}[Linear combination of base models \citep{modi2019sample}]\label{example:basemodel}
For an MDP $M(\cS, \cA, \gamma, \reward, \PP)$, suppose there exist $m$ base transition probability functions $\{p_i(s'|s,a)\}_{i=1}^{m}$, a feature mapping $\bpsi(s,a): \cS \times \cA \rightarrow \Delta^{d'}$ where $\Delta^{d'}$ is a $(d'-1)$-dimensional simplex, and an \emph{unknown} matrix $\Wb \in \RR^{m \times d'} \in [0,1]^{m\times d'}$ such that $\PP(s'|s,a) = \sum_{k=1}^{m} [\Wb\bpsi(s,a)]_k p_k(s'|s,a)$. Then it is a special case of linear kernel MDPs with feature mapping and parameter vector defined as follows: $d = md',\ \bphi(s'|s,a) = \text{vec}(\pb(s'|s,a)\bpsi(s,a)^\top)\in \RR^d,\ \btheta = \text{vec}(\Wb)\in \RR^d$, 
where $\text{vec}(\cdot)$ is the vectorization operator, and $\pb(s'|s,a) = [p_k(s'|s,a)] \in \RR^{m}$. 
\end{example}

\begin{example}[Feature embedding of a transition model \citep{yang2019reinforcement}] 
For an MDP $M(\cS, \cA, \gamma, \reward, \PP)$, suppose that there exist feature mappings $\bpsi_1(s,a): \cS \times \cA \rightarrow \RR^{d_1}$ satisfying $\|\bpsi_1(s,a)\|_2 \leq \sqrt{d_1}$, $\bpsi_2(s'): \cS \rightarrow \RR$ satisfying for any $\vvalue: \cS \rightarrow [0, R]$, $\|\sum_s \vvalue(s) \bpsi_2(s)\|_2 \leq R$ and an \emph{unknown} matrix $\Mb \in \RR^{d_1 \times d_2}$ satisfying $\|\Mb\|_F \leq \sqrt{d_1}$ such that $\PP(s'|s,a) = \bpsi_1(s,a)^\top \Mb\bpsi_2(s')$. Then it is a special case of linear kernel MDPs with the following feature mapping and parameter vector $d = d_1d_2,\ \bphi(s'|s,a) = \text{vec}\big(\bpsi_2(s')\bpsi_1(s,a)^\top\big) \in \RR^d,\ \btheta = \text{vec}(\Mb) \in \RR^d$.
\end{example}
\noindent\textbf{Comparison with linear MDPs.}  \citet{yang2019sample, jin2019provably} studied the so-called linear additive model or \emph{linear MDP}, which assumes the probability transition function can be represented as $\PP(\cdot|s,a) = \la \bpsi(s,a), \bmu(\cdot)\ra$, where $\bpsi(s,a)$ is a known feature mapping and $\mu(\cdot)$ is an unknown measure. It is worth noting that linear kernel MDPs studied in our paper and linear MDPs \citep{yang2019sample, jin2019provably} are two different classes of MDPs since they are based on different feature mappings, i.e., $\bphi(s'|s,a)$ versus $\bpsi(s,a)$. One cannot be covered by the other. For instance, some MDPs only fit linear MDPs such as $\PP(s'|s,a) = \sum_{i=1}^d \phi_i(s,a) \mu_i(s')$ satisfying $\phi_i(s,a)>0, \sum_{i=1}^d \phi_i(s,a) = 1$ and $\mu_i(s')$ is an unknown measure of $s'$. Some MDPs only fit linear kernel MDPs such as $\cS = \RR$, $\cA = \RR/\{0\}$, $\PP(s'|s,a) = \sum_{i=1}^d \theta_i p_i(s'|s,a)$, $p_i(s'|s,a) = \exp(-(s'-s-i)^2/(2a^2))/\sqrt{2\pi a^2}$. It is not a linear MDP because $p_i(s'|s,a)$ can not be decomposed as $\phi_i(s,a)\cdot\mu_i(s')$. In the rest of this paper, we assume the underlying linear kernel MDP is parameterized by $\btheta^*$ and denote it by $M_{\btheta^*}$.


In the online learning setting, the environment picks the starting state $s_1$ at the beginning. The goal is to design a nonstationary policy $\pi$ such that the expected discounted return at step $t$, $\vvalue^\pi_t(s_t)$, is close to the optimal expected return $\vvalue^*(s_t)$. We formalize this goal as minimizing the regret, which can be defined as follows, inspired by \citet{liu2020regret}. 
\begin{definition}\label{def:regret} 
For any policy $\pi$, we define its regret on MDP $M(\cS, \cA, \gamma, \reward, \PP)$ in the first $T$ rounds as the sum of the suboptimality $\Delta_t$ for $t = 1,\ldots, T$, i.e.,
\begin{align}
    \text{Regret}(\pi, M, T) = \sum_{t=1}^T \Delta_t,\ \text{where} \ \Delta_t = 
    \vvalue^*(s_t) - \vvalue^\pi_t(s_t),\notag
\end{align}
\end{definition}
Due to the optimality of the optimal value function $\vvalue^*$, we know that $\Delta_t \geq 0$ for any policy $\pi$. This fact suggests that $\text{Regret}(T)$ can be regarded as a cumulative error for $\pi$ to learn the optimal value function of MDP $M$. 

\noindent\textbf{Relation to sample complexity of exploration.}
A related quantity widely used for discounted MDPs is called the \emph{sample complexity of exploration} $N(\epsilon, \delta)$ \citep{szita2010model, lattimore2012pac, dong2019q}, which is defined as the number of rounds $t$ where $\Delta_t$ is greater than $\epsilon$ with probability at least $1-\delta$.
Note that algorithms with smaller regret make fewer mistakes in total, but they could make several severe mistakes. In comparison, algorithms with smaller sample complexity of exploration do not make severe mistakes, but they may suffer from a large number of mistakes in total. Therefore, these two quantities are not directly comparable. For any algorithm with $\tilde O(C\epsilon^{-a})$ sample complexity of exploration, where $C$ is a problem dependent constant (e.g., it may depend on $|\cS|,|\cA|,\gamma,d$), we can do a conversion and show that the algorithm also enjoys a $\tilde O(C^{1/(a+1)}(1-\gamma)^{-1/(a+1)} T^{a/(a+1)})$ regret for the first $T$ rounds. The proof is deferred to Appendix~\ref{app:conversion}.
More comparisons and discussions can also be found in \citet{liu2020regret} for the tabular setting. 
 

\section{The Proposed Algorithm} \label{section 4}

\begin{algorithm*}[t]
	\caption{Upper-Confidence Linear Kernel Reinforcement Learning ($\algname$)}\label{algorithm}
	\begin{algorithmic}[1]
	\REQUIRE Regularization parameter $\lambda$, confidence radius $\beta$, number of value iteration rounds $U$, time horizon $T$
	\STATE	Receive $s_1$ 
	\STATE Set $t \leftarrow 1$, $\bSigma_1 \leftarrow \lambda\Ib$, $\bbb_1 = \zero$
	\FOR{$k=0,\ldots$}
	\STATE Set $t_k \leftarrow t$, $\hat\btheta_k \leftarrow \bSigma_{t_k}^{-1}\bbb_{t_k}$\alglinelabel{algorithm:line1}
\STATE Set $\cC_k$ and $\qvalue_k(\cdot, \cdot)$ as follows:
\begin{align}
\cC_k = \{\btheta:\|\bSigma_{t_k}^{1/2}(\btheta - \hat\btheta_{k})\|_2 \leq \beta\},\ \qvalue_k(\cdot, \cdot)\leftarrow \valueite(\cC_k, U)\notag
\end{align}
\STATE Set $\vvalue_k(\cdot) \leftarrow \max_{a \in \cA}\qvalue_k(\cdot,a)$\alglinelabel{algorithm:line2}
	\REPEAT\alglinelabel{algorithm:line3}
	\STATE Set $\pi_t(\cdot)\leftarrow \argmax_a \qvalue_k(\cdot, a)$, take action $a_t  \leftarrow \pi_t(s_t)$, receive $s_{t+1} \sim \PP(\cdot|s_t, a_t)$\alglinelabel{algorithm:line99}
	\STATE Set $\bSigma_{t+1} \leftarrow \bSigma_{t} + \bphi_{{\vvalue}_k}(s_t,a_t)\bphi_{{\vvalue}_k}(s_t,a_t)^\top$
	\STATE Set $\bbb_{t+1} \leftarrow \bbb_{t} + \bphi_{\vvalue_k}(s_t, a_t)\vvalue_k(s_{t+1})$
\STATE $t \leftarrow t+1$
	\UNTIL{$\text{det}(\bSigma_{t}) > 2\text{det}(\bSigma_{t_k})$}\alglinelabel{algorithm:line4}
	\ENDFOR
	\end{algorithmic}
\end{algorithm*}

\begin{algorithm}[t]
	\caption{Extended Value Iteration: $\valueite(\cC, U)$}\label{algorithm:2}
	\begin{algorithmic}[1]
	\REQUIRE Confidence set $\cC$, number of value iteration rounds $U$
\STATE Let $\qvalue^{(0)}(\cdot,\cdot ) = 1/(1-\gamma)$.
\STATE $Q(\cdot,\cdot) \leftarrow Q^{(0)}(\cdot,\cdot)$
\IF {$\cC \cap \cB\neq \emptyset$}
\FOR{$u=1,\ldots, U$}
	\STATE Let $\vvalue^{(u-1)}(\cdot) = \max_{a \in \cA}\qvalue^{(u-1)}(\cdot, a)$ and
	\begin{align}
	    \qvalue^{(u)}(\cdot, \cdot)\leftarrow  \reward(\cdot, \cdot) + \gamma \max_{\btheta \in \cC\cap \cB} \big\la \btheta, \bphi_{\vvalue^{(u-1)}}(\cdot, \cdot)\big\ra\label{eq:onestepv}
	\end{align}
	\ENDFOR
	\STATE Let $\qvalue(\cdot, \cdot) \leftarrow \qvalue^{(U)}(\cdot, \cdot)$
\ENDIF
	\ENSURE $\qvalue(\cdot,\cdot)$
	\end{algorithmic}
\end{algorithm}

In this section, we propose an algorithm namely $\algname$ to learn the linear kernel MDP, which is illustrated in Algorithm \ref{algorithm}. $\algname$ is essentially a multi-epoch algorithm inspired by \citet{jaksch2010near, lattimore2012pac}. Specifically, the $k$-th epoch of Algorithm \ref{algorithm} starts at round $t_k$ and ends at round $t_{k+1}-1$. The length of each epoch is not prefixed but depends on previous observations. 
In each epoch, $\algname$ uses Extended Value Iteration ($\valueite$) function to compute the estimated optimal action-value function $\qvalue_k$ and selects the greedy policy according to the function.
The reason for using adaptive epoch length is that it can control the amount of ``switching error'' which occurs when the policy is updated.  Each epoch of $\algname$ can be divided into two phases, which we will discuss in detail in the sequel.

\noindent\textbf{Planning phase (Line \ref{algorithm:line1} to \ref{algorithm:line2})} 
 Planning phase is executed at the beginning of each epoch. In this phase, $\algname$ first computes $\hat\btheta_k$ as the estimate of $\btheta^*$, which is the minimizer of the following regularized least-square problem: 
  \begin{align}
    \hat\btheta_{k} &\leftarrow \argmin_{\btheta \in \RR^d}\sum_{j=0}^{k-1}\sum_{i=t_j}^{t_{j+1}-1} \big[\big\la \btheta, \bphi_{{\vvalue}_j}(s_i,a_i)\big\ra - {\vvalue}_j(s_{i+1})\big]^2 + \lambda\|\btheta\|_2^2,\label{def:leastsquare}
\end{align}
which has a closed-form solution as shown in Line \ref{algorithm:line1}. 
Then Algorithm \ref{algorithm} computes the confidence set of $\btheta^*$ as $\cC_k$ based on the confidence radius parameter $\beta$. Based on the confidence set $\cC_k$, Algorithm~\ref{algorithm} uses Algorithm \ref{algorithm:2} to compute the action-value functions $\qvalue_k$ for the next step. 


\noindent\textbf{Extended value iteration} 
Algorithm \ref{algorithm} makes use of $\valueite$ in Algorithm \ref{algorithm:2} to compute the action-value function corresponding to the near-optimal MDP among all the plausible MDPs $\cM_k$ induced by $\cC_k$.
In detail, besides $\cC_k$, $\valueite$ needs to access an additional set $\cB$ defined as follows:
\begin{align}
    \cB = \Big\{\btheta: \forall (s,a),\ \la \bphi(\cdot|s,a), \btheta\ra\text{ is a probability distribution}\Big\}.\notag
\end{align}
The intuition of introducing set $\cB$ is that since $\btheta^* \in \cB$, then $\cC_k \cap \cB$ is a tighter confidence set of $\btheta^*$.  In addition, $\cB$ is a convex set since it is easy to verify that: for any $\btheta_1,\btheta_2 \in \cB$, and any $\alpha \in [0,1]$, we have $\alpha\btheta_1 + (1-\alpha)\btheta_2$ belongs to $\cB$. $\cB$ contains all possible $\btheta^*$, which can be uniquely decided by the MDP class $\cM$. For instance, when $\cM$ is the global convex combination MDP class \citep{modi2019sample}, $\cB$ is a $d$-dimensional simplex. 
At each iteration of Algorithm~\ref{algorithm:2}, to obtain the new action-value function $\qvalue^{(u)}$, $\valueite$ performs one-step optimal value iteration \eqref{eq:onestepv} by selecting the best possible MDP $\tilde M$ among $\cM$ to maximize the Bellman backup over the previous value function $\vvalue^{(u-1)}$. This can be illustrated as follows:
\begin{align}
	    \qvalue^{(u)}(\cdot, \cdot)
	    &\leftarrow \reward(\cdot, \cdot) + \gamma \max_{\btheta \in \cC \cap \cB} \big\la \btheta, \bphi_{\vvalue^{(u-1)}}(\cdot, \cdot)\big\ra= \reward(\cdot, \cdot) + \gamma \max_{\tilde M \in \cM} \big[\tilde \PP\vvalue^{(u-1)}\big](\cdot, \cdot) .\notag
\end{align}
$\valueite$ returns the last action-value function as its output and sets $\qvalue_k = \qvalue^{(U)}$. 

\noindent\textbf{Execution phase (Line \ref{algorithm:line3} to \ref{algorithm:line4})} Execution phase is used to execute the policy in each epoch, collect observations, and update parameters. At round $t$, Algorithm \ref{algorithm} follows the greedy policy $\pi_t$ induced by $\qvalue_k$ to take the action $\pi_t(s_t)$ and observes the new state $s_{t+1}$. Algorithm \ref{algorithm} then computes vector $\bphi_{\vvalue_k}(s_t, a_t)$ according to Definition \ref{def:bphi} and the value function at $s_{t+1}$, i.e., $\vvalue_k(s_{t+1})$. Next, Algorithm \ref{algorithm} updates parameters $\bSigma_t$ and $\bbb_t$ by $\bphi_{\vvalue_k}(s_t, a_t)$. The loop repeats until 
$\text{det}(\bSigma_{t})>2\text{det}(\bSigma_{t_k})$. This is the same as the stopping criterion used by Rarely Switching OFUL in \citet{abbasi2011improved}. 

\noindent\textbf{Implementation of Algorithms \ref{algorithm} and \ref{algorithm:2}} 
There are two main implementation issues in Algorithms \ref{algorithm} and \ref{algorithm:2}. First, we need to compute the integration $\bphi_{\vvalue}(s,a)$ efficiently. Second, for Algorithm \ref{algorithm:2}, we need to compute $\qvalue(\cdot, \cdot)$ from $\valueite$ efficiently. Both of them can be efficiently achieved by Monte Carlo integration in some special cases, and we deferred the details to the appendix. Finally, it is worth noting that $\algname$ is an online reinforcement learning algorithm as it does not need to store all the past observations. $\algname$ only needs to maintain a vector $\bbb_t$ and a matrix $\bSigma_t$, which costs $O(d^2)$ space complexity. 



\section{Main Theory} \label{section 5}
In this section, we provide the theoretical analysis of Algorithm \ref{algorithm}. We introduce a shorthand notation $\text{Regret}(T)$ for $\text{Regret}(\algname, M_{\btheta^*}, T)$, when there is no confusion.
We present our main theorem, which gives an upper bound of the regret for Algorithm \ref{algorithm}.
\begin{theorem}\label{thm:regret}
Let $M_{\btheta^*}$ be the underlying linear kernel MDP. If we set $\beta$ and $U$ in Algorithm \ref{algorithm} as follows:
\begin{align}
    &\beta = \frac{1}{1-\gamma}\sqrt{d\log\frac{\lambda(1-\gamma)^2+Td}{\delta\lambda(1-\gamma)^2}} + \sqrt{\lambda d},\ U = \bigg\lceil\frac{\log (T/(1-\gamma))}{1-\gamma} \bigg\rceil\label{eq:betau},
\end{align}
then with probability at least $1-2\delta$, we have
\begin{align}
   \text{Regret}(T) &\leq \frac{6\beta}{1-\gamma}\sqrt{dT\log\frac{\lambda+T/(1-\gamma)^2}{\lambda}}+ \frac{5}{(1-\gamma)^2}\notag \\
    & \qquad + \frac{3\sqrt{T\log 1/\delta}}{(1-\gamma)^2} +\frac{3d}{(1-\gamma)^2}\log \frac{2\lambda +Td}{\lambda(1-\gamma)^2} .\label{eq: finalregret}
\end{align}
\end{theorem}

Theorem \ref{thm:regret} suggests that the regret of Algorithm \ref{algorithm} is in the order of $\tilde O(d\sqrt{T}/(1-\gamma)^2)$.

\begin{remark}
Several aspects of Theorem \ref{thm:regret} are worth to comment. Thanks to the feature mapping $\bphi$ and the multi-epoch nature of Algorithm \ref{algorithm}, the regret bound \eqref{eq: finalregret} in Theorem \ref{thm:regret} is independent of $|\cS|$ and $|\cA|$, which suggests that $\algname$ is sample efficient even for MDPs with large state and action spaces. This is in sharp contrast to the tabular RL algorithms, whose regret bound or sample complexity depends on $|\cS|$ and $|\cA|$ polynomially. Moreover, the exploration parameter $\beta$ and the number of extended value iteration rounds $U$ depend on $T$ logarithmically. For the case where $T$ is unknown, we can use the  ``doubling trick'' \citep{besson2018doubling} to learn $T$ adaptively, and it will only increase the regret \eqref{eq: finalregret} by a constant factor.  
\end{remark}

\begin{remark}
For the tabular MDPs, $\algname$ uses the feature mapping in Example \ref{exp: tabular} with a $|\cS|^2|\cA|$-dimension feature mapping. In that case, $\algname$ has a $|\cS|^2|\cA|\sqrt{T}/(1-\gamma)^2$ regret according to Theorem \ref{thm:regret}, which is worse than that of \citet{liu2020regret} considering the dependence of $|\cS|$ and $|\cA|$. There is no contradiction here, as in this paper, we aim to deliver a generic RL algorithm for linear kernel MDPs, which is a strictly larger class of MDPs than tabular MDPs. In fact, the regret bound in Theorem \ref{thm:regret} can be improved by providing a tighter confidence set $\cC_k$ specialized to the tabular MDP case. This is beyond the focus of this paper, and we leave it in the future work. 
\end{remark}

In addition to the upper bound result, we also prove the lower bound result. The following theorem shows a lower bound for any algorithm to learn a linear kernel MDP. 
\begin{theorem}\label{thm:1}
Suppose $\gamma \geq 2/3, d \geq 2$ and $T \geq \max\{d^2/225, 5\gamma\}/(1-\gamma)$. Then for any policy $\pi$, there exists a linear kernel MDP $M_{\tilde\btheta}$ such that
\begin{align}
    \EE \big[\text{Regret}(\pi, M_{\tilde\btheta}, T)\big] \geq \frac{\gamma d\sqrt{T}}{1600c(1-\gamma)^{1.5}}- \frac{\gamma}{(1-\gamma)^2}.\label{eq:finallower}
\end{align}
\end{theorem}
\begin{remark}
Theorem \ref{thm:1} suggests that when $T$ is large enough, the lower bound of regret \eqref{eq:finallower} is $\Omega(d\sqrt{T}/(1-\gamma)^{1.5})$. Compared with the upper regret bound $\tilde O(d\sqrt{T}/(1-\gamma)^2)$, we can conclude that $\algname$ has an optimal dependence on the feature mapping dimension $d$ and the time horizon $T$, and the dependence on the discount factor is only worse than the lower bound by a $(1-\gamma)^{-0.5}$ factor. 
\end{remark}

\section{Proof Sketch of the Main Theory}\label{sec:sketch}
In this section, we provide the proof sketches of the upper and lower bounds on the regret. The complete proofs are deferred to the appendix.
\subsection{Proof Sketch of Theorem \ref{thm:regret}}\label{sec:sketchregret}
In this section we prove Theorem \ref{thm:regret}. Let $K(T)-1$ be the number of epochs when Algorithm \ref{algorithm} executes $t = T$ rounds, and $t_{K(T)} = T+1$. We have the following technical lemmas.

\begin{lemma}\label{lemma:theta-ball}
Let $\beta$ be defined in \eqref{eq:betau}. Then with probability at least $1-\delta$, for all $0 \leq k\leq K(T)-1$, we have $\cC_k \cap \cB$ is non-empty and $\btheta^* \in \cC_k \cap \cB$.
\end{lemma}
Lemma \ref{lemma:theta-ball} suggests that in every epoch of Algorithm \ref{algorithm}, $\btheta^*$ is contained in the confidence sets $\{\cC_k\cap \cB\}_{k=0}^{K(T)-1}$ with a high probability.

\begin{lemma}\label{lemma:UCB}
Let the event in Lemma \ref{lemma:theta-ball} hold. Then for all $0 \leq k \leq K(T)-1$, we have $1/(1-\gamma) \geq \qvalue_k(s,a) \geq \qvalue^*(s,a)$ for any $(s,a) \in \cS \times \cA$. 
\end{lemma}
Lemma \ref{lemma:UCB} suggests that in every epoch of Algorithm \ref{algorithm}, $\qvalue_k(s,a)$ found by $\valueite$ is an upper bound for the optimal action-value function $\qvalue^*(s,a)$.  

Recall that the goal of $\valueite$ is to find the action-value function $\qvalue_k$ corresponding to the optimal MDP in $\cM_k$, which should satisfy the following optimality condition
\begin{align}
    \qvalue_k(s_t,a_t) = \reward(s_t,a_t) + \gamma\max_{\btheta \in \cC_k\cap \cB}\big\la \btheta, \bphi_{\vvalue_k}(s_t, a_t)\big\ra .\notag
\end{align}
However, it is impossible to find the exactly optimal value function since $\valueite$ only performs finite number of iterations. The following lemma characterizes the error of $\valueite$ after $U$ iterations. 

\begin{lemma}\label{lemma:V-diff}
Let the event in Lemma \ref{lemma:theta-ball} hold. Then for any $0 \leq k \leq K(T)-1$ and $t_k \leq t \leq t_{k+1}-1$, there exists a $\btheta_t \in \cC_k \cap \cB$ such that $\qvalue_k(s_t, a_t) \leq \reward(s_t, a_t) + \gamma \big\la \btheta_t, \bphi_{\vvalue_k}(s_t, a_t)\big\ra + 2\gamma^U$.

\end{lemma}
 Lemma~\ref{lemma:V-diff} 
 suggests that for any $\epsilon>0$, $\valueite$ in Algorithm \ref{algorithm:2} only needs to perform $\log(1/\epsilon)$ iterations to achieve an $\epsilon$-suboptimal action-value function.

\begin{lemma}\label{lemma:boundk}
We have $K(T) \leq 2d\log[(\lambda +dT)/(\lambda(1-\gamma)^2)]$. 
\end{lemma}
Lemma \ref{lemma:boundk} suggests that Algorithm \ref{algorithm} only needs to update its policy for $K(T) = \tilde O(d)$ times, which is almost independent of the time horizon $T$. In sharp contrast, RL algorithms with feature mapping in the finite-horizon setting need to update their policy every $H$ steps \citep{jin2019provably, modi2019sample}, which leads to $O(T/H)$ number of updates.
\begin{proof}[Proof sketch of Theorem \ref{thm:regret}]
The regret can be decomposed as follows:
 \begin{align}
    \text{Regret}(T) &= \sum_{k=0}^{K(T)-1}\sum_{t=t_k}^{t_{k+1}-1} \big[\vvalue^*(s_t) - \vvalue^\pi_t(s_t)\big] \leq\sum_{k=0}^{K(T)-1}\underbrace{\sum_{t=t_k}^{t_{k+1}-1} \big[\vvalue_k(s_t)-\vvalue^\pi_t(s_t)\big]}_{E_k},\label{eq:0}
  \end{align}
where the inequality holds  
due to Lemma \ref{lemma:UCB}. $E_k$ can be further bounded as follows by Bellman equation and Lemma~\ref{lemma:V-diff}. 
\begin{align}
    E_k
    &\leq 2/(1-\gamma)^2 + 2\gamma^U(t_{k+1} - t_k)/(1-\gamma) +  \sum_{t = t_k}^{t_{k+1}-1}\la \btheta_t- \btheta^*, \bphi_{\vvalue_k}(s_t, a_t)\ra/(1-\gamma) + \Xi_t,\label{eq:sketch_1}
\end{align}
where $\Xi_t = \big[ \big[\PP(\vvalue_k-\vvalue^\pi_{t+1})\big](s_t,a_t)-\big(\vvalue_k(s_{t+1})-\vvalue^\pi_{t+1}(s_{t+1})\big)\big]/(1-\gamma)$. 
Taking summation of \eqref{eq:sketch_1} from $k=0$ to $K(T)-1$ and rearranging it, we obtain that $\sum_{k=0}^{K(T)-1}E_k$ is upper bounded as follows
\begin{align}
        \sum_{k=0}^{K(T)-1} E_k
    &\leq \frac{2K(T}{(1-\gamma)^2} + \frac{2\gamma^UT}{1-\gamma}+ \sum_{k=0}^{K(T)-1}\sum_{t = t_k}^{t_{k+1}-1}\Xi_t  + \sum_{k=0}^{K(T)-1} \sum_{t = t_k}^{t_{k+1}-1}\frac{\la \btheta_t- \btheta^*, \bphi_{\vvalue_k}(s_t, a_t)\ra}{1-\gamma} ,\notag
\end{align}
where the first term on the R.H.S. can be further bounded by $\tilde O(d/(1-\gamma)^2)$ by Lemma \ref{lemma:boundk}, the second term can be  bounded by 1 with the choice of $U$, the third term can be bounded by $\tilde O(d\sqrt{T}/(1-\gamma)^2)$ by Lemma \ref{lemma:theta-ball}, and the last term can be bounded by $\tilde O(\sqrt{T}/(1-\gamma)^2)$ by Azuma-Hoeffding inequality. 
\end{proof}

\subsection{Proof Sketch of Theorem \ref{thm:1}}\label{subsec:lower bound}


    



\begin{figure*}[!h]
    \centering
    \begin{tikzpicture}[node distance=2cm,>=stealth',bend angle=45,auto]

  \tikzstyle{place}=[circle,thick,draw=blue!75,fill=blue!20,minimum size=6mm]
\tikzset{every loop/.style={min distance=15mm, font=\footnotesize}}
  \begin{scope}[xshift=-3.5cm]
    \node [place, label=left:{...}] (c1)                                    {$\state_1$};

    \coordinate[left of=c1] (e1) {}; 
    \node [place, label=right:{...}, label=left:{...}] (e2) [left of=e1] {$\state_0$}; 
    \path[->] (e2)
    edge [in=170,out=110, min distance = 6cm, loop] node[above] {$1- \delta - \la \ab_1, \btheta\ra$} ()
    edge [in=120,out=60] node[below]{$\delta + \la \ab_1, \btheta\ra$} (c1)
    
        edge [in=-170,out=-110, min distance = 6cm, loop] node[below] {$1- \delta - \la \ab_i, \btheta\ra$} ()
    edge [in=-120,out=-60] node[above]{$\delta + \la \ab_i, \btheta\ra$} (c1)
    ;

  \end{scope}
    \begin{scope}[xshift=2cm]
    \node [place] (c1)                                    {$\state_1$};

    \coordinate[left of=c1] (e1) {}; 
    \node [place] (e2) [left of=e1] {$\state_0$}; 
    \path[->] (c1)
    edge [in=-60,out=60, min distance = 6cm, loop] node[right] {$1- \delta$} ()
    edge [in=0,out=180] node[auto]{$\delta$} (e2)
    ;

  \end{scope}

\end{tikzpicture}
\vspace*{-8mm}
    \caption{Class of hard-to-learn linear kernel MDPs considered in Section \ref{subsec:lower bound}. The left figure demonstrates the state transition probability starting from $\state_0$ with different action $\ab_i$. The right figure demonstrates the state transition probability starting from $\state_1$ with any action.}
    \label{fig:hardmdp}
\end{figure*}
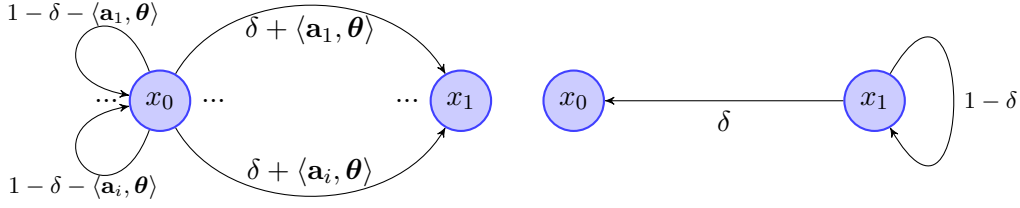

At the core of the proof of Theorem \ref{thm:1} is to construct a class of hard-to-learn MDP instances. We show the construction of these instances here and defer the detailed proof to Appendix \ref{app:main}.
Let $M(\cS, \cA, \gamma, \reward, \PP_{\btheta})$ denote these hard MDPs. 
The state space $\cS$ consists of two states $\state_0, \state_1$. The action space $\cA$ consists of $2^{d-1}$ vectors $\ab \in \{-1, 1\}^{d-1}$. The reward function $\reward$ satisfies that $\reward(\state_0, \ab) = 0$ and $\reward(\state_1, \ab) = 1$ for any $\ab \in \cA$. The probability transition function $\PP_{\btheta}$ is parameterized by a $(d-1)$-dimensional vector $\btheta \in \bTheta$, $\bTheta = \{ -\Delta/(d-1), \Delta/(d-1)\}^{d-1}$, which is defined as $\PP_{\btheta}(\state_0|\state_0, \ab) = 1-\delta - \la \ab, \btheta\ra$, $\PP_{\btheta}(\state_1|\state_0, \ab) = \delta + \la \ab, \btheta\ra$, $\PP_{\btheta}(\state_0|\state_1, \ab) = \delta$, $\PP_{\btheta}(\state_1|\state_1, \ab) = 1-\delta$, 
where $\delta$ and $\Delta$ are positive parameters that need to be determined in later proof. 
It can be verified that $M$ is indeed a linear kernel MDP with the vector $\tilde\btheta = ( \btheta^\top, 1)^\top \in \RR^d$ while $\Delta \leq d-1$ and the feature mapping $\bphi(s'|s,a)$ defined as follows:
\begin{align}
    &\bphi(\state_0|\state_0,\ab) = \begin{pmatrix}
    -\ab \\
    1-\delta
    \end{pmatrix}, \bphi(\state_1|\state_0,\ab)=  \begin{pmatrix}
    \ab \\
    \delta
    \end{pmatrix}, \bphi(\state_0|\state_1,\ab)=  \begin{pmatrix}
    \zero \\
    \delta
    \end{pmatrix}, 
    \bphi(\state_1|\state_1,\ab)=  \begin{pmatrix}
    \zero\\
    1-\delta
    \end{pmatrix}.\notag
\end{align}
\begin{remark}
The class of hard-to-learn linear kernel MDPs can be regarded as an extension of the hard instance in linear bandits literature \citep{dani2008stochastic, lattimore2018bandit} to MDPs. 
Our constructed MDPs are similar to those in \citet{jaksch2010near, osband2016lower} for the average-reward MDPs and \citet{lattimore2012pac} for the discounted MDPs. By Example \ref{exp: tabular}, we know that tabular MDPs can be regarded as specialized linear kernel MDPs with a $|\cS|^2|\cA|$-dimensional feature mapping. However, simply applying the MDPs in \citet{jaksch2010near, osband2016lower, lattimore2012pac} to our setting would yield a $\Omega(\sqrt{|\cS||\cA|T}/(1-\gamma)^{1.5})$ lower bound for regret, which is looser than our result because $\sqrt{|\cS||\cA|} \leq |\cS|^2|\cA| = d$. 
\end{remark}
From now on, we set $\delta = 1-\gamma,\ \Delta = d\sqrt{1-\gamma}/(90\sqrt{2T})$ and only consider the case where $\pi$ is a deterministic policy, since the regret result of the case where $\pi$ is stochastic is lower bounded by that of the deterministic one. Let $N_0$ denote the total visit number to state $\state_0$. Similiarily, let $N_1$ denote the total visit number to state $\state_1$, $N_0^{\ab}$ denote the total visit number to state $\state_0$ followed by action $\ab$ and $N_0^{\tilde \cA}$ denote the total visit number to state $\state_0$ followed by actions in subset $\tilde \cA \subseteq \cA$. Let $\cP_{\btheta}(\cdot)$ denote the distribution over $\cS^T$, where $s_1 = \state_0$, $s_{t+1} \sim \PP_{\btheta}(\cdot|s_t, a_t)$, $a_t$ is decided by $\pi_t$. Let $\EE_{\btheta}$ denote the expectation w.r.t. distribution $\cP_{\btheta}$. Suppose we have an MDP $M(\cS, \cA, \gamma, \reward, \PP_{\btheta})$. During this proof the starting state $s_1$ is set to be $\state_0$. 
For simplicity, let $\text{Regret}(\btheta)$ denote $\text{Regret}(\pi, M(\cS, \cA, \gamma, \reward, \PP_{\btheta}), T)$ without confusion. We need the following lemmas. The first lemma shows that to bound $\text{Regret}(\btheta)$, we only need to bound the summation of rewards over $s_t, a_t$.
\begin{lemma}\label{lemma:regrettrans}
The regret $\text{Regret}(\btheta)$ satisfies that
\begin{align}
    &\EE_{\btheta}\text{Regret}(\btheta) \geq \EE_{\btheta}\bigg[\sum_{t=1}^T \Big[\vvalue^*(s_t) - \frac{1}{1-\gamma}\reward(s_t, a_t)\Big]-\frac{\gamma}{(1-\gamma)^2}\bigg].\notag
\end{align}
\end{lemma}
Next lemma gives the relation between $\EE_{\btheta}N_1$, $\EE_{\btheta}N_0^{\ab}$ and $\EE_{\btheta}N_0$, which is useful to our proof. 
\begin{lemma}\label{lemma:n1}
Suppose $2\Delta<\delta$ and $(1-\delta)/\delta < T/5$, then for $\EE_{\btheta}N_1$ and $\EE_{\btheta}N_0$, we have
\begin{align}
    &\EE_{\btheta}N_1 \leq \frac{T}{2} +\frac{1}{2\delta}\sum_{\ab}\la \ab, \btheta\ra \EE_{\btheta}N_0^{\ab},\ \text{and} \quad\EE_{\btheta}N_0 \leq 4T/5.\notag
\end{align}
\end{lemma}
Next lemma gives the bound for KL divergence. 
\begin{lemma}\label{lemma:kl}
Suppose that $\btheta$ and $\btheta'$ only differs from $j$-th coordinate, $2\Delta<\delta \leq 1/3 $. Then we have the following bound for the KL divergence between $\cP_{\btheta}$ and $\cP_{\btheta'}$:
\begin{align}
    \text{KL}(\cP_{\btheta'}\| \cP_{\btheta}) \leq \frac{16\Delta^2}{(d-1)^2\delta}\EE_{\btheta}N_0.\notag
\end{align}
\end{lemma}
\begin{proof}[Proof Sketch of Theorem \ref{thm:1}]
By Lemma \ref{lemma:regrettrans}, we only need to lower bound the difference between $\vvalue^*$ and $\reward(s_t, a_t)$. We can calculate $\vvalue^*$ through the definition of our MDP as
\begin{align}
    &\vvalue^*(\state_0) = \frac{\gamma(\Delta + \delta)}{(1-\gamma)(\gamma(2\delta+\Delta - 1) +1)},\ \vvalue^*(\state_1) = \frac{\gamma(\Delta + \delta) + 1-\gamma}{(1-\gamma)(\gamma(2\delta+\Delta - 1) +1)}.\notag
\end{align}
Since $\reward(\state_0, \ab) = 0$ and $\reward(\state_1, \ab) = 1$, then the lower bound can be fully characterized by $\EE_{\btheta}N_1$. Furthermore, we can derive that
\begin{align}
    \frac{1}{|\bTheta|}\sum_{\btheta}\EE_{\btheta}N_1 &\leq \frac{T}{2}+\frac{1}{4\delta}\frac{\Delta}{(d-1)|\bTheta|}\sum_{j=1}^{d-1}\sum_{\btheta} \Big[\EE_{\btheta'}N_0+ \frac{cT}{8}\sqrt{\text{KL}( \cP_{\btheta'}\| \cP_{\btheta})}\Big],\label{eq:ttt}
\end{align}
where $\btheta'$ only differs from $\btheta$ at $j$-th coordinate. By Lemma \ref{lemma:n1} and \ref{lemma:kl} we can obtain an upper bound of \eqref{eq:ttt} in terms of $\delta$ and $\Delta$. Selecting $\delta = 1-\gamma,\ \Delta = d\sqrt{1-\gamma}/(90\sqrt{2T})$ gives the final result. 
\end{proof}

\section{Conclusion}\label{section 8}
We proposed a novel algorithm for solving linear kernel MDPs called $\algname$. We prove that the regret of $\algname$ can be upper bounded by $\tilde O(d\sqrt{T}/(1-\gamma)^2)$, which is the first result of its kind for learning discounted MDPs without accessing the generative model or making strong assumptions like uniform ergodicity. We also proved a lower bound $\Omega(d\sqrt{T}/(1-\gamma)^{1.5})$ which holds for any algorithm. There still exists a gap of $(1-\gamma)^{-0.5}$ between the upper and lower bounds, and we leave it as an open problem for future work.

\section*{Acknowledgement}
We thank Lin Yang for pointing out the related work \citep{jia2020model, ayoub2020model}, and Csaba Szepesv\'ari for helpful comments.

\appendix 
\section{Conversion from Sample Complexity to Regret}\label{app:conversion}
Suppose that an algorithm has $\tilde O(C\epsilon^{-a})$ sample complexity of exploration, where $C$ is a constant that may depend on the problem-dependent parameters such as $|\cS|,|\cA|,\gamma,d$. Then with probability at least $1-\delta$, it has at most $\tilde O(C \epsilon^{-a})$ number of rounds $t$ such that $\vvalue^*(s_t) - \vvalue^\pi_t(s_t) \geq \epsilon$. We denote the collection of these rounds by set $\cD$. Then for the first $T$ rounds, with probability at least $1-\delta$, its regret can be bounded as
\begin{align}
    \text{Regret}(T) &=  \sum_{t=1}^T \big[\vvalue^*(s_t) - \vvalue^\pi_t(s_t)\big]\notag \\
    &= \sum_{t \in [T] \cap \cD}\big[\vvalue^*(s_t) - \vvalue^\pi_t(s_t)\big] + \sum_{t \in [T] \setminus \cD}\big[\vvalue^*(s_t) - \vvalue^\pi_t(s_t)\big]\notag \\
    & \leq |\cD|\cdot 1/(1-\gamma) + T \cdot \epsilon\notag \\
    & = \tilde O(C \epsilon^{-a}/(1-\gamma) + \epsilon T),\notag
\end{align}
where the first inequality uses the fact that $0\leq V^*(s_t), V_t^\pi(s_t) \leq 1/(1-\gamma)$, and the last line holds due to the definition of $\cD$. 
Select $\epsilon = T^{-1/(a+1)}(1-\gamma)^{1/(a+1)}C^{-1/(a+1)}$ to minimize the above regret bound, we have $\text{Regret}(T) = \tilde O(C^{1/(a+1)}(1-\gamma)^{-1/(a+1)} T^{a/(a+1)})$. For example, if the sample complexity of exploration is $\tilde O(C\epsilon^{-2})$, then it implies an $\tilde O(C^{1/3} (1-\gamma)^{-1/3}T^{2/3})$ regret bound.

\section{Details of Implementation}
In this section, we discuss how to efficiently implement Algorithm \ref{algorithm} and Algorithm \ref{algorithm:2} by using Monte Carlo integration. We consider a special case similar to that of \citet{yang2019reinforcement}, where 
\begin{align}
    [\bphi(s'|s,a)]_j = [\bpsi(s')]_j\cdot [\bmu(s,a)]_j,\ |[\bmu(s,a)]_j| \leq 1,\ \Big|\sum_{s'} [\bpsi(s')]_j\Big| \leq D,\label{eq:imple}
\end{align}
where $D>0$ is a constant. \eqref{eq:imple} suggests the feature mapping $\bphi(s'|s,a)$ is the element-wise product of two feature mappings $\bpsi(s)$ and $\bmu(s,a)$. We consider the case where $|\cA|$ is finite. There are two main implementation issues in Algorithm \ref{algorithm} and Algorithm \ref{algorithm:2}. First, we need to compute the integration $\bphi_{\vvalue}(s,a)$ efficiently. Note that under \eqref{eq:imple}, the $j$-th coordinate of the integration $[\bphi_{\vvalue}(s,a)]_j$ can be decomposed into the production of $\sum_{s'}\vvalue(s')[\bpsi(s')]_j$ and $[\bmu(s,a)]_j$. Therefore, we can use Monte Carlo integration to evaluate $\sum_{s'}\vvalue(s')[\bpsi(s')]_j$ and obtain a \emph{uniform} accurate estimation for all $(s,a)$ \emph{simultaneously}. We have the following proposition which can be proved by using Azuma-Hoeffding inequality:
\begin{proposition}\label{prop:int}
Let $\vvalue$ be some $1/(1-\gamma)$-bounded function. Suppose for any $j \in [d]$, we have the access to the integration constant $I_j = \sum_{s' \in \cS}[\bpsi(s')]_j$. Then we generate $s^{i,j}$ and denote $\hat\bphi_{\vvalue}(s,a) \in \RR^d$ as follows:
\begin{align}
    s^{i, j} \sim \frac{[\bpsi(\cdot)]_j}{I_j},\ i =1,\dots, R,\ [\hat\bphi_{\vvalue}(s, a)]_{j} = I_{j}\cdot[\bmu(s, a)]_{j}\cdot\frac{1}{R}\sum_{i = 1}^R\vvalue(s^{i, j}),\notag
\end{align}
then for any $j$, with probability at least $1-\delta$, for all $(s,a) \in \cS \times \cA$, we have
\begin{align}
    \big|[\bphi_{\vvalue}(s, a)]_{j} - [\hat\bphi_{\vvalue}(s, a)]_{j}\big| \leq \frac{D\log(1/\delta)}{\sqrt{R}(1-\gamma)}.\notag
\end{align}
\end{proposition}
Thus, we can approximate $\bphi_{\vvalue}(s, a)$ up to $\epsilon$-accuracy by $\hat\bphi_{\vvalue}(s, a)$ using $\tilde O(1/\epsilon^2)$ points. Second, we consider the efficiency of $\valueite$. At the first glance we may need to store all values of $\qvalue$ over all $(s,a) \in \cS\times\cA$, which leads to a $|\cS| |\cA|$ space complexity. Actually the complexity can be greatly reduced by approximately using Monte Carlo integration as follows. We first randomly sample $URd$ data points $s^{u,i,j}, u \in [U], i \in [R], j \in [d]$ by $s^{\cdot, \cdot, j} \sim [\bpsi(\cdot)]_j/I_j$. At each iteration $u \leq U-1$, we calculate the values $\vvalue^{(u)}(s^{u,i,j})$ based on $\vvalue^{(u-1)}(s^{u-1,i',j'})$ through the following induction rule: 
\begin{align}
    &\vvalue^{(u)}(s^{u,i, j}) = \max_{a \in \cA}\bigg\{\reward(s^{u,i,j}, a) +\gamma \max_{\btheta \in \cC \cap \cB} \bigg\la \btheta, \hat\bphi_{\vvalue^{(u-1)}}(s^{u,i,j}, a)\bigg\ra \bigg\},\label{def:max} \\
    &[\hat\bphi_{\vvalue^{(u-1)}}(s^{u,i,j}, a)]_{j'} = I_{j'}\cdot[\bmu(s^{u,i,j}, a)]_{j'}\cdot\frac{1}{R}\sum_{i' = 1}^R\vvalue^{(u-1)}(s^{u-1, i', j'}).\notag
\end{align}
The maximization problem \eqref{def:max} is reduced to a constrained maximization problem over the convex set $\cC \cap \cB$, which can be solved by projected gradient methods \citep{boyd2004convex} efficiently in practice. Then at $U$-th iteration, we calculate $\qvalue^{(U)}(s,a)$ as
\begin{align}
    &\qvalue^{(U)}(s,a) = \reward(s, a) +\gamma \max_{\btheta \in \cC \cap \cB} \bigg\la \btheta, \hat\bphi_{\vvalue^{(U-1)}}(s, a)\bigg\ra,\notag \\
    &[\hat\bphi_{\vvalue^{(u-1)}}(s, a)]_{j} = I_j\cdot [\bmu(s, a)]_{j}\cdot \frac{1}{R}\sum_{i = 1}^R\vvalue^{(U-1)}(s^{U-1, i, j}).\notag
\end{align}
We can see that to calculate $Q^{(U)}(s,a)$, only $dR$ function values $\vvalue^{(U-1)}(s^{U-1, i, j}),i\in [R], j \in [d]$ need to be stored. Through the same argument of Proposition \ref{prop:int}, $Q^{(U)}(s,a)$ achieves $\epsilon$-accuracy using $R \sim \tilde O(1/\epsilon^2)$ samples. Finally, we analyze the computational complexity of EVI. Suppose we need $B$ time complexity to solve the maximization problem $\max_{\btheta \in \cC \cap \cB}\la \btheta, \ab\ra$ for any $\ab \in \RR^d$. Since we need to solve $d R |\cA|$ number of maximization problem at each iteration of EVI, then we need $UdRB|\cA|$ time complexity to obtain $\vvalue^{(U-1)}(s^{U-1, i, j}),i\in [R], j \in [d]$. After obtaining $\vvalue^{(U-1)}(s^{U-1, i, j})$, we need $dB|\cA|$ time complexity to calculate $\qvalue^{(U)}(s,a)$ for any $(s,a) \in \cS \times \cA$.

\section{Proof of Main Theory}\label{app:main}

In this section we provide the proof of main theory. 

\subsection{Proof of Theorem \ref{thm:regret}}\label{2:main}
In this subsection, we prove Theorem \ref{thm:regret}.  
Besides Lemmas \ref{lemma:theta-ball}-\ref{lemma:boundk} in Section \ref{sec:sketchregret}, we also need the following three additional lemmas.
\begin{lemma}[Azuma–Hoeffding inequality]\label{lemma:azuma}
Let $\{X_k\}_{k=0}^{\infty}$ be a
discrete-parameter real-valued martingale sequence such that for every $k\in \NN$, the condition $|X_k-X_{k-1}|\leq \mu$ holds for some non-negative constant $\mu$. Then with probability at least $1-\delta$, we have 
\begin{align}
    |X_n-X_0|\leq 2\mu\sqrt{n \log 1/\delta}.\notag
\end{align} 
\end{lemma}

\begin{lemma}[Lemma 11 in \citet{abbasi2011improved}]\label{lemma:sumcontext}
For any $\{\xb_t\}_{t=1}^T \subset \RR^d$ satisfying that $\|\xb_t\|_2 \leq L$, let $\Ab_0 = \lambda \Ib$ and $\Ab_t = \Ab_0 + \sum_{i=1}^{t-1}\xb_i\xb_i^\top$, then we have
\begin{align}
    \sum_{t=1}^T \min\{1, \|\xb_t\|_{\Ab_{t-1}^{-1}}\}^2 \leq 2d\log\frac{d\lambda+TL^2}{d \lambda}.\notag
\end{align}
\end{lemma}

\begin{lemma}[Lemma 12 in \citet{abbasi2011improved}]\label{lemma:det}
Suppose $\Ab, \Bb\in \RR^{d \times d}$ are two positive definite matrices satisfying that $\Ab \succeq \Bb$, then for any $\xb \in \RR^d$, $\|\xb\|_{\Ab} \leq \|\xb\|_{\Bb}\cdot \sqrt{\det(\Ab)/\det(\Bb)}$.
\end{lemma}

Now we are ready to prove Theorem \ref{thm:regret}. 

\begin{proof} [Proof of Theorem \ref{thm:regret}]
Let $K(T)-1$ be the number of epochs when Algorithm \ref{algorithm} executes $t = T$ rounds, and $t_{K(T)} = T+1$.
Suppose the event in Lemma \ref{lemma:theta-ball} holds. We have
 \begin{align}
    \text{Regret}(T) &= \sum_{k=0}^{K(T)-1}\sum_{t=t_k}^{t_{k+1}-1} \big[\vvalue^*(s_t) - \vvalue^\pi_t(s_t)\big]\leq\underbrace{\sum_{k=0}^{K(T)-1}\sum_{t=t_k}^{t_{k+1}-1} \big[\vvalue_k(s_t)-\vvalue^\pi_t(s_t)\big]}_{\text{Regret}'(T)},\label{eq:0}
  \end{align}
where the last inequality holds because of Lemma \ref{lemma:UCB}. For $\text{Regret}'(T)$, we have 
\begin{align}
 \text{Regret}'(T)&=\sum_{k=0}^{K(T)-1}\sum_{t=t_k}^{t_{k+1}-1} \Big[\qvalue_k(s_t,a_t)-\vvalue^\pi_t(s_t)\Big],\label{eq:1}
\end{align}
where the  equality holds because of the policy in Line \ref{algorithm:line99} in Algorithm \ref{algorithm}.
By Lemma \ref{lemma:V-diff}, with the selection of $U$, for $t_k\leq t\leq  t_{k+1}-1$, the $\qvalue_k(s_t,a_t)$ in Algorithm \ref{algorithm} satisfies
\begin{align}
    \qvalue_k(s_t,a_t)&\leq \reward(s_t,a_t) + \gamma  \big[\big\la \btheta_{t}, \bphi_{\vvalue_k}(s_t,a_t)\big\ra\big]_{[0, 1/(1-\gamma)]} + (1-\gamma)/T, \label{Qt-update}
\end{align}
By the Bellman equation and the fact that $a_t = \pi(s_t, t)$, 
we have
\begin{align}
    \vvalue^\pi_t(s_t)&=\reward(s_t,a_t)+\gamma[\PP \vvalue^\pi_{t+1}](s_t,a_t)\notag\\
    &=\reward(s_t,a_t)+\gamma \sum_{s' \in \cS} \big\la \btheta^*, \bphi(s'|s_t,a_t)\big \ra\vvalue^\pi_{t+1}(s') ds'\notag\\
    &=\reward(s_t,a_t) + \gamma  \big\la \btheta^*, \bphi_{\vvalue^\pi_{t+1}}(s_t,a_t)\big\ra,\label{Qpi-update}
\end{align}
where the second and the third equalities hold because of Definition \ref{assumption-linear}. Substituting \eqref{Qt-update} and \eqref{Qpi-update} into \eqref{eq:1}, we have
\begin{align}
    &\text{Regret}'(T)-(1-\gamma) \notag \\
    &\qquad\leq\gamma\sum_{k=0}^{K(T)-1}\sum_{t=t_k}^{t_{k+1}-1} \big(\big[\big\la \btheta_{t}, \bphi_{\vvalue_k}(s_t,a_t)\big\ra\big]_{[0, 1/(1-\gamma)]}-\big\la \btheta^*, \bphi_{\vvalue^\pi_{t+1}}(s_t,a_t)\big\ra\big)\notag\\
    &\qquad=\underbrace{\gamma \sum_{k=0}^{K(T)-1}\sum_{t=t_k}^{t_{k+1}-1}\big(\big[\big\la \btheta_{t}, \bphi_{\vvalue_k}(s_t,a_t)\big\ra\big]_{[0, 1/(1-\gamma)]}-\big\la \btheta^*, \bphi_{\vvalue_k}(s_t,a_t)\big\ra\big)}_{I_1} \notag \\
    &\qquad\quad  +\gamma\sum_{k=0}^{K(T)-1}\sum_{t=t_k}^{t_{k+1}-1} \big\la \btheta^*,\bphi_{\vvalue_k}(s_t,a_t)- \bphi_{\vvalue^\pi_{t+1}}(s_t,a_t)\big\ra\notag\\
    &\qquad=I_1+I_2+I_3,\label{eq:regret}
\end{align}
where
\begin{align}
    I_2 &= \gamma\sum_{k=0}^{K(T)-1}\sum_{t=t_k}^{t_{k+1}-1} \Big\{\big[\PP(\vvalue_k-\vvalue^\pi_{t+1})\big](s_t,a_t)-\big(\vvalue_k(s_{t+1})-\vvalue^\pi_{t+1}(s_{t+1})\big)\Big\},\notag \\
    I_3& = \gamma \sum_{k=0}^{K(T)-1}\sum_{t=t_k}^{t_{k+1}-1}\big(\vvalue_k(s_{t+1})-\vvalue^\pi_{t+1}(s_{t+1})\big).\notag
\end{align}
Next we bound $I_1, I_2$ and $I_3$ separately. 
For term $I_1$, we have
\begin{align}
        I_1&\leq\gamma \sum_{k=0}^{K(T)-1}\sum_{t=t_k}^{t_{k+1}-1}\big|\big\la \btheta_{t}-\btheta^*, \bphi_{\vvalue_k}(s_t,a_t)\big\ra\big|\notag\\
        &\leq  \sum_{k=0}^{K(T)-1}\sum_{t=t_k}^{t_{k+1}-1}\big(\big\|\btheta_{t}-\hat\btheta_k\big\|_{\bSigma_{t}}+\big\|\hat\btheta_k-\btheta^*\big\|_{\bSigma_{t}}\big)\|\bphi_{\vvalue_k}(s_t,a_t)\|_{\bSigma_{t}^{-1}}\notag\\
        & \leq 
        2\sum_{k=0}^{K(T)-1}\sum_{t=t_k}^{t_{k+1}-1}
        \big(\big\|\btheta_{t}-\hat\btheta_k\big\|_{\bSigma_{t_k}}+\big\|\hat\btheta_k-\btheta^*\big\|_{\bSigma_{t_k}}\big)\|\bphi_{\vvalue_k}(s_t,a_t)\|_{\bSigma_{t}^{-1}}\notag\\
        &\leq  4\beta \sum_{k=0}^{K(T)-1}\sum_{t=t_k}^{t_{k+1}-1} \|\bphi_{\vvalue_k}(s_t,a_t)\|_{\bSigma_{t}^{-1}},\label{eq:regret_11}
\end{align}
where the first inequality holds since $0 \leq \big\la \btheta^*, \bphi_{\vvalue_k}(s_t,a_t)\big\ra \leq 1/(1-\gamma)$, the second inequality holds due to the Cauchy-Schwarz inequality and triangle inequality, the third inequality holds due to Lemma \ref{lemma:det} with the fact that $\det(\bSigma_{t}) \leq 2\det(\bSigma_{t_k})$, and the fourth inequality holds due to the fact that $\btheta_t \in \cC_k$ from Lemma \ref{lemma:theta-ball}. Meanwhile, we have 
\begin{align}
    \big[\big\la \btheta_{t}, \bphi_{\vvalue_k}(s_t,a_t)\big\ra\big]_{[0, 1/(1-\gamma)]}-\la \btheta^*, \bphi_{\vvalue_k}(s_t,a_t)\big\ra  \leq \frac{1}{1-\gamma},\label{eq:regret_12}
\end{align}
where we use the fact that $0 \leq \vvalue^* \leq 1/(1-\gamma)$. 
Combining \eqref{eq:regret_11} and \eqref{eq:regret_12}, $I_1$ can be further bounded as
\begin{align}
    I_1 &\leq \sum_{k=0}^{K(T)-1}\sum_{t=t_k}^{t_{k+1}-1}\min\bigg\{\frac{1}{1-\gamma},  4\beta  \|\bphi_{\vvalue_k}(s_t,a_t)\|_{\bSigma_{t}^{-1}}\bigg\}\notag \\
    &\leq 4\beta \sum_{k=0}^{K(T)-1}\sum_{t=t_k}^{t_{k+1}-1}\min\bigg\{1, \|\bphi_{\vvalue_k}(s_t,a_t)\|_{\bSigma_{t}^{-1}}\bigg\}\notag \\
    & \leq 4\beta\sqrt{T\sum_{k=0}^{K(T)-1}\sum_{t=t_k}^{t_{k+1}-1}\min\bigg\{1, \|\bphi_{\vvalue_k}(s_t,a_t)\|_{\bSigma_{t}^{-1}}^2\bigg\}},\label{eq:i_1_1}
\end{align}
where the second inequality holds because $1/(1-\gamma)\leq \beta$, the last inequality holds due to Cauchy-Schwarz inequality. By Lemma \ref{lemma:sumcontext}, we have
\begin{align}
    \sum_{k=0}^{K(T)-1}\sum_{t=t_k}^{t_{k+1}-1}\min\bigg\{1, \|\bphi_{{\vvalue}_{k}}(s_t,a_t)\|_{\bSigma_{t}^{-1}}^2\bigg\} \leq 2d\log\frac{\lambda+T/(1-\gamma)^2}{\lambda},\label{eq:i_1_2}
\end{align}
where we use the fact $\|\bphi_{\vvalue_k}(s_t,a_t)\|_2 \leq \sqrt{d}/(1-\gamma)$ deduced by Definition \ref{assumption-linear} and $|\vvalue_k|\leq 1/(1-\gamma)$ implied by Lemma \ref{lemma:UCB}. 
Substituting \eqref{eq:i_1_2} into \eqref{eq:i_1_1}, we have
\begin{align}
    I_1 \leq 6\beta\sqrt{dT\log\frac{\lambda+T/(1-\gamma)^2}{\lambda}}.\label{eq:i_1_3}
\end{align}
For the term $I_2$, it is easy to verify that $\big[\PP(\vvalue_k-\vvalue^\pi_{t+1})\big](s_t,a_t)-\big(\vvalue_k(s_{t+1})-\vvalue^\pi_{t+1}(s_{t+1})\big)$ forms a martingale difference sequence. Meanwhile, we have $0\leq \vvalue_k(s)-\vvalue^\pi_{t+1}(s)\leq 1/(1-\gamma)$ implied by Lemma \ref{lemma:UCB}, which implies that
\begin{align}
    \bigg|\big[\PP(\vvalue_k-\vvalue^\pi_{t+1})\big](s_t,a_t)-\big(\vvalue_k(s_{t+1})-\vvalue^\pi_{t+1}(s_{t+1})\big)\bigg|\leq \frac 1 {1-\gamma}.\notag
\end{align}
Thus by Azuma–Hoeffding inequality in Lemma \ref{lemma:azuma}, we have
\begin{align}
    I_2=\gamma\sum_{k=0}^{K(T)-1}\sum_{t=t_k}^{t_{k+1}-1} \big[\PP(\vvalue_k-\vvalue^\pi_{t+1})\big](s_t,a_t)-\big(\vvalue_k(s_{t+1})-\vvalue^\pi_{t+1}(s_{t+1})\big)\leq  \frac{2\gamma}{1-\gamma}\sqrt{T\ln \frac{1}{\delta}}.\label{eq:I_2}
\end{align}
For the term $I_3$, we have
\begin{align}
    I_3&=\gamma \sum_{k=0}^{K(T)-1}\sum_{t=t_k}^{t_{k+1}-1}\big(\vvalue_k(s_{t+1})-\vvalue^\pi_{t+1}(s_{t+1})\big)\notag\\
    & = \gamma \sum_{k=0}^{K(T)-1}\bigg[\sum_{t=t_k}^{t_{k+1}-1}\big(\vvalue_k(s_{t})-\vvalue^\pi_t(s_{t})\big) - \big(\vvalue_k(s_{t_k})-\vvalue^\pi_{t_k}(s_{t_k})\big)\notag \\
    &\qquad +\big(\vvalue_k(s_{t_{k+1}})-\vvalue^\pi_{t_{k+1}}(s_{t_{k+1}})\big)\bigg]\notag\\
    & \leq \gamma \sum_{k=0}^{K(T)-1}\bigg[\sum_{t=t_k}^{t_{k+1}-1}\big(\vvalue_k(s_{t})-\vvalue^\pi_t(s_{t})\big) + \frac{2}{1-\gamma}\bigg]\notag\\
    &= \gamma \text{Regret}'(T)+\frac{2K(T)\gamma}{1-\gamma},\label{eq:I_3:1}
\end{align}
where the first inequality holds due to $0\leq \vvalue_k(s)-\vvalue^\pi_{t}(s)\leq 1/(1-\gamma) $ implied by Lemma \ref{lemma:UCB}. 
Finally, substituting \eqref{eq:i_1_3}, \eqref{eq:I_2} and \eqref{eq:I_3:1} into \eqref{eq:regret}, we have 
\begin{align}
    &\text{Regret}'(T) -(1-\gamma)\notag \\
    & \qquad\leq 6\beta\sqrt{dT\log\frac{\lambda+T/(1-\gamma)^2}{\lambda}} + \frac{2\gamma}{1-\gamma}\sqrt{T\ln \frac{1}{\delta}} + \gamma \text{Regret}'(T)+\frac{2K(T)\gamma}{1-\gamma}.\label{eq:regret_2}
\end{align}
Thus, we have
\begin{align}
    \text{Regret}'(T) 
    &\leq \frac{6\beta}{1-\gamma}\sqrt{dT\log\frac{\lambda+T/(1-\gamma)^2}{\lambda}} + \frac{2\gamma}{(1-\gamma)^2}\sqrt{T\ln \frac{1}{\delta}} + \frac{2K(T)\gamma}{(1-\gamma)^2}+1.\label{eq:regret_3}
\end{align}
Substituting $\beta$ and \eqref{eq:regret_3} into \eqref{eq:0} and rearranging it, we have
\begin{align}
    \text{Regret}(T) &\leq \frac{6}{1-\gamma}\sqrt{dT\log\frac{\lambda+T/(1-\gamma)^2}{\lambda}}\bigg(\frac{1}{1-\gamma}\sqrt{d\log\frac{\lambda(1-\gamma)^2+Td}{\delta\lambda(1-\gamma)^2}} + \sqrt{\lambda d}\bigg)\notag \\
    & \qquad + \frac{2\gamma}{(1-\gamma)^2}\sqrt{T\ln \frac{1}{\delta}}  +1+\frac{2K(T)\gamma}{(1-\gamma)^2}\notag \\
    & \leq \frac{6}{1-\gamma}\sqrt{dT\log\frac{\lambda+T/(1-\gamma)^2}{\lambda}}\bigg(\frac{1}{1-\gamma}\sqrt{d\log\frac{\lambda(1-\gamma)^2+Td}{\delta\lambda(1-\gamma)^2}} + \sqrt{\lambda d}\bigg)\notag \\
    & \qquad + \frac{3\sqrt{T\log 1/\delta}}{(1-\gamma)^2} + 1 +\frac{4d}{(1-\gamma)^2}\log \frac{\lambda +Td}{\lambda(1-\gamma)^2}\notag,
\end{align}
where the last inequality holds due to Lemma \ref{lemma:boundk} and the fact that $U = \lceil (\log (T/(1-\gamma))/(1-\gamma)\rceil$. Taking an union bound of Lemma \ref{lemma:theta-ball} and Lemma \ref{lemma:azuma}, we conclude the proof.

\end{proof}

\subsection{Proof of Theorem \ref{thm:1} }\label{sec:proof:thm1}
In this subsection, we will prove Theorem \ref{thm:1}. Besides Lemmas \ref{lemma:regrettrans}-\ref{lemma:kl} in Section \ref{subsec:lower bound}, we need the following additional technical lemma, which is a version of Pinsker's inequality adapted from \citet{jaksch2010near} that upper bounds the total variation distance between two signed measure in terms of the Kullback–Leibler (KL) divergence.
\begin{lemma}[Pinsker's inequality]\label{lemma:pinsker}
Denote $\sbb = \{s_1,\dots,s_T\} \in \cS^{T}$ as the observed states from step $1$ to $T$. Then for any two distributions $\cP_1$ and $\cP_2$ over $\cS^{T}$ and any bounded function $f: \cS^{T}\rightarrow [0, B]$, we have
\begin{align}
    \EE_1 f(\sbb) - \EE_2 f(\sbb) \leq \sqrt{\log 2/2}B\sqrt{\text{KL}(\cP_2\|\cP_1)},\notag
\end{align}
where $\EE_1$ and $\EE_2$ denote expectations with respect to $\cP_1$ and $\cP_2$. 
\end{lemma}

Now we begin our proof. The proof roadmap is similar to that in \citet{jaksch2010near} which aims to prove lower bound for tabular MDPs. 

\begin{proof}[Proof of Theorem \ref{thm:1}]
First, we can verify that all assumptions in Lemmas \ref{lemma:n1} and \ref{lemma:kl} are satisfied with the assumptions on $\gamma$ and $T$ and the choice of $\delta$ and $\Delta$. For a given $\btheta$, the optimal policy for $M(\cS, \cA, \gamma, \reward, \PP_{\btheta})$ is to choose action 
$\ab_{\btheta} = [\text{sgn}(\theta_i)]_{i=1}^{d-1}$ at $\state_0$ and $\state_1$. Therefore by the optimality Bellman equation, we know that $\vvalue^*(\state_0)$ and $\vvalue^*(\state_1)$ satisfy the following equations
\begin{align}
    \vvalue^*(\state_0) = \reward(\state_0, \ab_{\btheta}) + \gamma \EE_{s \sim \PP_{\btheta}(\cdot|\state_0, \ab_{\btheta})}\vvalue^*(s),\ \vvalue^*(\state_1) = \reward(\state_1, \ab_{\btheta}) + \gamma \EE_{s \sim \PP_{\btheta}(\cdot|\state_1, \ab_{\btheta})}\vvalue^*(s).\label{eq:thm1_00}
\end{align}
By the definition of our MDP, we have $\reward(\state_0, \ab_{\btheta}) = 0$, $\reward(\state_1, \ab_{\btheta}) = 1$, and \begin{align}
    &\PP_{\btheta}(\state_0|\state_0, \ab_{\btheta}) = 1-\delta - \la \ab_{\btheta}, \btheta\ra = 1-\delta - \Delta, \notag \\
    &\PP_{\btheta}(\state_1|\state_0, \ab_{\btheta}) = \delta + \la \ab_{\btheta}, \btheta\ra = \delta + \Delta, \notag\\
    &\PP_{\btheta}(\state_0|\state_1, \ab) = \delta,\notag \\ &\PP_{\btheta}(\state_1|\state_1, \ab) = 1-\delta.\notag
\end{align}
Therefore, substituting the above definitions of $\reward$ and $\PP_{\btheta}$ into \eqref{eq:thm1_00}, we have the following equations.
\begin{align}
\begin{split}
    &\vvalue^*(\state_0) = 0 + \gamma\cdot (1-\delta -\Delta)\vvalue^*(\state_0) + \gamma \cdot (\delta + \Delta)\vvalue^*(\state_1),\\
    &\vvalue^*(\state_1) = 1 + \gamma \cdot \delta \vvalue^*(\state_0) + \gamma \cdot (1-\delta)\vvalue^*(\state_1). 
    \end{split}\label{eq:thm1_01} 
\end{align}
From \eqref{eq:thm1_01}, we can calculate $\vvalue^*(\state_0)$ and $\vvalue^*(\state_1)$ as follows
\begin{align}\label{eq:optimalv}
\vvalue^*(\state_0) = \frac{\gamma(\Delta + \delta)}{(1-\gamma)(\gamma(2\delta+\Delta - 1) +1)},\  \vvalue^*(\state_1) = \frac{\gamma(\Delta + \delta) + 1-\gamma}{(1-\gamma)(\gamma(2\delta+\Delta - 1) +1)}.
\end{align}
Then by Lemma \ref{lemma:regrettrans} we have
\begin{align}
    \EE_{\btheta}\text{Regret}(\btheta) 
    & \geq \EE_{\btheta}\bigg[ \sum_{t=1}^T\vvalue^*(s_t) - \frac{1}{1-\gamma}\sum_{t' = 1}^T \reward(s_{t'}, a_{t'}) - \frac{\gamma}{(1-\gamma)^2}\bigg].\notag
\end{align}
Now we do the summation over $2^{d-1}$ possible $\btheta$, then the expectation of regret can be written as follows:
\begin{align}
    &\frac{1}{|\bTheta|}\sum_{\btheta}\bigg[\EE_{\btheta} \text{Regret}(\btheta) +\frac{\gamma}{(1-\gamma)^2}\bigg]\notag \\
    & \qquad\geq 
    \frac{1}{|\bTheta|}\sum_{\btheta}\EE_{\btheta}\bigg[N_0\vvalue^*(\state_0) + N_1\vvalue^*(\state_1) - \frac{1}{1-\gamma}\sum_{ t= 1}^T \reward(s_{t}, a_{t}) \bigg]\notag \\
    &\qquad=  \frac{1}{1-\gamma}\frac{1}{|\bTheta|}\sum_{\btheta}\EE_{\btheta}\bigg[N_0\frac{\gamma(\Delta+\delta)}{\gamma(2\delta+\Delta - 1) +1} + N_1\bigg(\frac{\gamma(\Delta+\delta)+1-\gamma}{\gamma(2\delta+\Delta - 1) +1} - 1\bigg)\bigg]\notag \\
    & \qquad= \frac{1}{1-\gamma}\frac{1}{|\bTheta|}\sum_{\btheta}\EE_{\btheta}\bigg[N_0\frac{\gamma(\Delta+\delta)}{\gamma(2\delta+\Delta - 1) +1} + N_1\frac{-\gamma\delta}{\gamma(2\delta+\Delta - 1) +1}\bigg]\notag\\
    & \qquad= \frac{1}{1-\gamma}\frac{1}{|\bTheta|}\sum_{\btheta}\EE_{\btheta}\bigg[T\frac{\gamma(\Delta+\delta)}{\gamma(2\delta+\Delta - 1) +1} - N_1\frac{\gamma(\Delta+2\delta)}{\gamma(2\delta+\Delta - 1) +1}\bigg]\notag\\
    & \qquad= \frac{1}{1-\gamma}T\frac{\gamma(\Delta+\delta)}{\gamma(2\delta+\Delta - 1) +1} - \frac{1}{1-\gamma}\frac{\gamma(\Delta+2\delta)}{\gamma(2\delta+\Delta - 1) +1}\frac{1}{|\bTheta|}\sum_{\btheta}\EE_{\btheta}N_1,\label{eq:998}
\end{align}
where the first equality holds due to the value of $\vvalue^*(\state_0)$, $\vvalue^*(\state_1)$ in \eqref{eq:optimalv} and the fact that $\reward(s_t, a_t) = 1$ for $s_t = \state_1$ and $\reward(s_t, a_t) = 0$ for $s_t = \state_0$. Next we are going to bound $|\bTheta|^{-1}\sum_{\btheta}\EE_{\btheta}N_1$. By Lemma \ref{lemma:n1}, we have
\begin{align}
    \frac{1}{|\bTheta|}\sum_{\btheta}\EE_{\btheta}N_1 &\leq T/2 + \frac{1}{2\delta|\bTheta|}\sum_{\btheta}\sum_{\ab}\la \ab, \btheta\ra \EE_{\btheta}N_0^{\ab}\notag \\
    & = \frac{T}{2}+\frac{1}{2\delta}\frac{\Delta}{(d-1)|\bTheta|}\sum_{j=1}^{d-1}\sum_{\ab}\sum_{\btheta}\EE_{\btheta} (2\ind\{\text{sgn}(\theta_j) =  \text{sgn}(a_j)\}-1) N_0^{\ab}\notag\\
    & \leq \frac{T}{2}+\frac{1}{2\delta}\frac{\Delta}{(d-1)|\bTheta|}\sum_{j=1}^{d-1}\sum_{\ab}\sum_{\btheta}\EE_{\btheta} \ind\{\text{sgn}(\theta_j) =  \text{sgn}(a_j)\} N_0^{\ab},\label{eq:999}
\end{align}
where the second inequality holds since $(2\ind\{\text{sgn}(\theta_j) =  \text{sgn}(a_j)\}-1) \leq \ind\{\text{sgn}(\theta_j) =  \text{sgn}(a_j)\}$. From now on we focus on some specific $j \in [d-1]$.
Taking $\btheta'$ to be the vector which has the same entries as $\btheta$, only except for $j$-th coordinate such as $\theta'_j = -\theta_j$. Then
\begin{align}
&\EE_{\btheta} \ind\{\text{sgn}(\theta_j) =  \text{sgn}(a_j)\} N_0^{\ab} + \EE_{\btheta'} \ind\{\text{sgn}(\theta'_j) =  \text{sgn}(a_j)\} N_0^{\ab}\notag \\
& \qquad= \EE_{\btheta'}N_0^{\ab}+\EE_{\btheta} \ind\{\text{sgn}(\theta_j) =  \text{sgn}(a_j)\} N_0^{\ab} - \EE_{\btheta'} \ind\{\text{sgn}(\theta_j) =  \text{sgn}(a_j)\} N_0^{\ab}.\label{eq:tttt}
\end{align}
Thus taking summation of \eqref{eq:tttt} for all $\ab \in \cA$ and $\btheta \in \bTheta$, we have
\begin{align}
    &2\sum_{\ab}\sum_{\btheta}\EE_{\btheta} \ind\{\text{sgn}(\theta_j) =  \text{sgn}(a_j)\} N_0^{\ab}\notag \\
    &\qquad = \sum_{\btheta}\sum_{\ab}\Big[\EE_{\btheta'}N_0^{\ab}+ \EE_{\btheta}\ind\{\text{sgn}(\theta_j) =  \text{sgn}(a_j)\} N_0^{\ab} - \EE_{\btheta'}\ind\{\text{sgn}(\theta_j) =  \text{sgn}(a_j)\} N_0^{\ab}\Big]\notag \\
    & \qquad= \sum_{\btheta}\Big[\EE_{\btheta'}N_0+ \EE_{\btheta}N_0^{\cA_j^{\btheta}} - \EE_{\btheta'}N_0^{\cA_j^{\btheta}}\Big]\notag\\
    & \qquad\leq \sum_{\btheta}\Big[\EE_{\btheta'}N_0+ \frac{cT}{8}\sqrt{\text{KL}( \cP_{\btheta'}\| \cP_{\btheta})}\Big]\notag\\
    &\qquad \leq \sum_{\btheta}\Big[\EE_{\btheta'}N_0+ \frac{cT\Delta}{d\sqrt{\delta}}\sqrt{\EE_{\btheta }N_0}\Big]\label{eq:100},
\end{align}
where $\cA_j^{\btheta}$ is the set of $\ab$ which satisfies that  $\text{sgn}(\theta_j) = \text{sgn}(a_j)$, $c = 4\sqrt{\log 2}$. The first inequality holds due to Lemma \ref{lemma:pinsker} with the fact that $N_0^{\cA_j^{\btheta}}$ is a function of $s_1,\dots,s_T$ and    $N_0^{\cA_j^{\btheta}} \leq T$, the second inequality holds due to Lemma \ref{lemma:kl}. Substituting \eqref{eq:100} into \eqref{eq:999}, we have
\begin{align}
    \frac{1}{|\bTheta|}\sum_{\btheta}\EE_{\btheta}N_1 &\leq T/2 + \frac{\Delta}{4\delta(d-1)|\bTheta|}\sum_{j=1}^{d-1}\sum_{\btheta}\Big[\EE_{\btheta'}N_0+ cT\frac{\Delta}{d\sqrt{\delta}}\sqrt{\EE_{\btheta }N_0}\Big]\notag \\
    & = T/2 + \frac{\Delta}{4\delta|\bTheta|}\sum_{\btheta}\Big[\EE_{\btheta'}N_0+ cT\frac{\Delta}{d\sqrt{\delta}}\sqrt{\EE_{\btheta }N_0}\Big]\notag \\
    & \leq \frac{T}{2} + \frac{\Delta T}{5\delta}+ \frac{cT^{3/2}\Delta^2}{4d\delta^{3/2}},\label{eq:1001}
\end{align}
where the last inequality holds due to $\EE_{\btheta}N_0,\EE_{\btheta'}N_0 \leq 4T/5$ from Lemma \ref{lemma:n1}. Substituting \eqref{eq:1001} into \eqref{eq:998}, we have
\begin{align}
    &\frac{1}{|\bTheta|}\sum_{\btheta}\bigg[\EE_{\btheta} \text{Regret}(\btheta) +\frac{\gamma}{(1-\gamma)^2}\bigg]\notag \\
    & \qquad\geq  \frac{1}{1-\gamma}T\frac{\gamma(\Delta+\delta)}{\gamma(2\delta+\Delta - 1) +1} - \frac{1}{1-\gamma}\frac{\gamma(\Delta+2\delta)}{\gamma(2\delta+\Delta - 1) +1} \cdot\bigg(\frac{T}{2} + \frac{\Delta T}{5\delta}+ \frac{cT^{3/2}\Delta^2}{4d\delta^{3/2}}\bigg)\notag \\
    & \qquad=\frac{1}{(1-\gamma)(\gamma(2\delta+\Delta - 1) +1)}\bigg[\frac{\gamma \Delta T}{2}-\gamma(\Delta+2\delta)\frac{\Delta T}{5\delta} - \gamma(\Delta+2\delta)\frac{cT^{3/2}\Delta^2}{4d\delta^{3/2}}\bigg]\notag \\
    & \qquad\geq \frac{1}{4(1-\gamma)^2     }\bigg[\frac{\gamma \Delta T}{2}-\gamma(\Delta+2\delta)\frac{\Delta T}{5\delta} - \gamma(\Delta+2\delta)\frac{cT^{3/2}\Delta^2}{4d\delta^{3/2}}\bigg]\notag \\
    &\qquad \geq 
    \frac{1}{4(1-\gamma)^2     }\bigg[\frac{\gamma \Delta T}{2}-\gamma\frac{9\delta}{4}\frac{\Delta T}{5\delta} - \gamma\frac{9\delta}{4}\frac{cT^{3/2}\Delta^2}{4d\delta^{3/2}}\bigg]\notag \\
    &\qquad = \frac{1}{4(1-\gamma)^2}\bigg[\frac{1}{20}\gamma \Delta T - \gamma\frac{9cT^{3/2}\Delta^2}{16d\sqrt{\delta}} \bigg]\notag \\
    &\qquad= \frac{\gamma d\sqrt{T}}{1600c(1-\gamma)^{1.5}},\notag
\end{align}
where the second inequality holds since $\delta = 1-\gamma$ and $\gamma(2\delta+\Delta - 1) +1 \leq 1- \gamma + 3\delta\gamma = 1- \gamma + 3(1-\gamma)\gamma \leq 4(1-\gamma)$, the third inequality holds due to the fact that $4\Delta<\delta\leq 1/3$, the last inequality holds due to the choice of $\Delta$ and $\delta$. Therefore, there exists $\btheta \in \bTheta$ such that
\begin{align}
    \EE_{\btheta} \text{Regret}(\btheta) \geq \frac{\gamma d\sqrt{T}}{1600c(1-\gamma)^{1.5}}- \frac{\gamma}{(1-\gamma)^2}.\notag
\end{align}
Setting $\tilde\btheta = (\btheta^\top, 1)^\top \in \RR^d$ completes our proof. 
\end{proof}

\section{Proof of lemmas in Section \ref{sec:sketchregret}}

\subsection{Proof of Lemma \ref{lemma:theta-ball}}
\begin{proof}[Proof of Lemma \ref{lemma:theta-ball}]
Recall the definition of $\hat\btheta_k$ in Algorithm \ref{algorithm}, we have
\begin{align}
    \hat\btheta_{k}  = \bigg(\lambda \Ib+\sum_{j=0}^{k-1}\sum_{i=t_j}^{t_{j+1}-1}\bphi_{{\vvalue}_j}(s_i,a_i)\bphi_{{\vvalue}_j}(s_i,a_i)^\top\bigg)^{-1}\bigg(\sum_{j=0}^{k-1}\sum_{i=t_j}^{t_{j+1}-1} \bphi_{{\vvalue}_j}(s_i,a_i) \vvalue_j(s_{i+1})\bigg).\notag
\end{align}
It is worth noting that for any $0 \leq j \leq k-1$ and $t_j \leq i \leq t_{j+1}-1$,
\begin{align}
    [\PP\vvalue_j](s_i, a_i) &= \sum_{s'}\PP(s'|s_i, a_i)\vvalue_j(s_i, a_i) \notag \\
    &= \sum_{s'}\la \bphi(s'|s_i,a_i), \btheta^*\ra\vvalue_j(s')\notag \\
    & = \Big\la\sum_{s'} \bphi(s'|s_i,a_i)\vvalue_j(s'), \btheta^*\Big\ra\notag \\
    & = \la \bphi_{{\vvalue}_j}(s_i,a_i), \btheta^*\ra,
\end{align}
thus $\{\vvalue_j(s_{i+1}) - \la \bphi_{\vvalue_j}(s_i,a_i), \btheta^*\ra\}$ forms a martingale difference sequence. Besides, since $0 \leq \vvalue_j(s) \leq 1/(1-\gamma)$ for any $s$, then $\vvalue_j(s_{i+1}) - \la \bphi_{\vvalue_j}(s_i,a_i), \btheta^*\ra$ is a sequence of $1/(1-\gamma)$-subgaussian random variables with zero means. Meanwhile, we have $\|\bphi_{\vvalue_j}(s_i,a_i)\|_2 \leq \sqrt{d}/(1-\gamma)$ and $\|\btheta^*\|_2 \leq S$ by Definition \ref{assumption-linear}. 
By Theorem 2 in \citet{abbasi2011improved}, we have that with probability at least $1-\delta$, $\btheta^*$ belongs to the following set for all $1 \leq k \leq K$:
\begin{align}
    \bigg\{\btheta: \Big\|\bSigma_{t_k}^{1/2}(\btheta - \hat\btheta_k)\Big\|_2 \leq \frac{1}{1-\gamma}\sqrt{d\log\frac{\lambda(1-\gamma)^2+t_kd}{\delta\lambda(1-\gamma)^2}} + \sqrt{\lambda}S\bigg\}.
\end{align}
Finally, by the definition of $\beta_k$ and the fact that $\la \btheta^*, \bphi(s'|s,a)\ra = \PP(s'|s,a)$ for all $(s,a)$, we draw the conclusion that $\btheta^* \in \cB\cap\cC_k$ for $1 \leq k \leq K$. 
\end{proof}
\subsection{ Proof of Lemma \ref{lemma:UCB}}
\begin{proof} [Proof of Lemma \ref{lemma:UCB}]
    We use induction to prove this lemma.
    We only need to prove that for all $0 \leq u \leq U$, $\qvalue^{(u)} \geq \qvalue^*$. We have
    \begin{align}
        \frac 1 {1-\gamma} = \qvalue^{(0)}(s,a)\ge \qvalue^*(s,a),\notag
    \end{align}
    where the inequality holds due to the fact that $Q^*(s,a) \leq 1/(1-\gamma)$ caused by $0 \leq \reward(s,a) \leq 1$. Assume that the statement holds for $u$, then $\qvalue^{(u)}(s,a) \geq \qvalue^*(s,a)$, which leads to $\vvalue^{(u)}(s) \geq \vvalue^*(s)$. Furthermore, we have
    \begin{align}
        \qvalue^{(u+1)}(s,a) - \reward(s,a)&=\gamma \max_{\btheta \in \cB\cap \cC} \big\la \btheta, \bphi_{\vvalue^{(u)}}(s,a)\big\ra\geq \gamma  \big\la \btheta^*, \bphi_{\vvalue^{(u)}}(s,a)\big\ra  =  \gamma \PP \vvalue^{(u)}(s,a),\label{eq:ttt1}
    \end{align}
    where the inequality holds since $\btheta^*\in \cC \cap \cB$ for any $(s,a)\in \cS \times \cA$ due to Lemma \ref{lemma:theta-ball}. 
    We further have
    \begin{align}
        \qvalue^{(u+1)}(s,a)&=\reward(s,a) + \gamma \tilde\PP \vvalue^{(u)}(s,a)\leq 1+\frac{\gamma}{1-\gamma} = \frac{1}{1-\gamma},\notag
    \end{align}
    where $\tilde\PP$ is some distribution, the equality holds since $\btheta \in \cB$, 
 the inequality holds due to the fact that $\vvalue^{(u)}(s) \leq 1/(1-\gamma)$. We also have
    \begin{align}
       \qvalue^{(u+1)}(s,a)&\geq \reward(s,a)+\gamma[\PP {\vvalue}^{(u)}](s,a)\ge \reward(s,a)+\gamma[\PP {\vvalue}^*](s,a)=\qvalue^*(s,a),\notag
    \end{align}
    where the first inequality holds due to \eqref{eq:ttt1}, and the second inequality holds because the induction assumption. Thus the statement holds for $u+1$. Therefore, our conclusion holds. 
\end{proof}

        

\subsection{Proof of Lemma \ref{lemma:V-diff}}
\begin{proof}[Proof of Lemma \ref{lemma:V-diff}]
We first prove the following inequality:
\begin{align}
    \qvalue^{(U)}(s,a) - \qvalue^{(U-1)}(s,a) \leq 2\gamma^{U-1}.\label{eq:V-diff_0}
\end{align}
By the update rule in Algorithm \ref{algorithm:2}, for any $u \geq 2$, we have
\begin{align}
	    &\qvalue^{(u)}(s,a)=  \reward(s,a) + \gamma \max_{\btheta \in \cC\cap\cB} \big\la \btheta, \bphi_{\vvalue^{(u-1)}}(s,a)\big\ra,\notag\\
	    &\qvalue^{(u-1)}(s,a)=  \reward(s,a) + \gamma \max_{\btheta \in \cC\cap\cB} \big\la \btheta, \bphi_{\vvalue^{(u-2)}}(s,a)\big\ra.\notag
\end{align}
Thus for any $(s,a) \in \cS\times \cA$, we have
\begin{align}
    \Big|\qvalue^{(u)}(s,a) - \qvalue^{(u-1)}(s,a)\Big| &= \gamma \bigg|\max_{\btheta \in \cC\cap\cB} \big\la \btheta, \bphi_{\vvalue^{(u-1)}}(s,a)\big\ra - \max_{\btheta \in \cC\cap\cB} \big\la \btheta, \bphi_{\vvalue^{(u-2)}}(s,a)\big\ra\bigg|\notag \\
    & \leq \gamma \max_{\btheta \in \cC\cap\cB}\Big|\big\la \btheta, \bphi_{\vvalue^{(u-1)}}(s,a) - \bphi_{\vvalue^{(u-2)}}(s,a)\big\ra\Big|\label{eq:V-diff_0.5} \\
    & = 
    \gamma \Big|\big\la \tilde\btheta, \bphi_{\vvalue^{(u-1)}}(s,a) - \bphi_{\vvalue^{(u-2)}}(s,a)\big\ra\Big|\notag \\
    & = \gamma \Big|\tilde \PP\big[\vvalue^{(u-1)} - \vvalue^{(u-2)} \big](s,a)\Big|,
    \label{eq:V-diff_1}
\end{align}
where $\tilde\btheta$ is the $\btheta$ which attains the maximum of \eqref{eq:V-diff_0.5}, and $\tilde\PP(s'|s,a) = \la \tilde\btheta, \bphi(s'|s,a)\ra$. The inequality holds due to the contraction property of $\max$ function. Then \eqref{eq:V-diff_1} can be further bounded as follows:
\begin{align}
    \gamma \Big|\tilde \PP\big[\vvalue^{(u-1)} - \vvalue^{(u-2)} \big](s,a)\Big|& \leq \gamma\max_{s' \in \cS} \Big|\vvalue^{(u-1)}(s') - \vvalue^{(u-2)}(s')\Big|\notag \\
    & =  \gamma\max_{s' \in \cS} \Big|\max_{a'\in \cA}\qvalue^{(u-1)}(s', a') - \max_{a'\in \cA}\qvalue^{(u-2)}(s', a')\Big|\notag \\
    & \leq 
    \gamma \max_{(s', a') \in \cS \times \cA}\Big|\qvalue^{(u-1)}(s', a') - \qvalue^{(u-2)}(s', a')\Big|,\label{eq:V-diff_2}
\end{align}
where the first inequality holds due to the fact that $|\tilde \PP f(s,a)| \leq \max_{s' \in \cS} |f(s')|$ for any $(s,a,s')$, the second inequality holds due to the contraction property of max function. Substituting \eqref{eq:V-diff_2} into \eqref{eq:V-diff_1} and taking the maximum over $(s,a)$, we have
\begin{align}
    \max_{(s,a) \in \cS \times \cA}\Big|\qvalue^{(u)}(s,a) - \qvalue^{(u-1)}(s,a)\Big| \leq \gamma \max_{(s,a) \in \cS \times \cA}\Big|\qvalue^{(u-1)}(s,a) - \qvalue^{(u-2)}(s,a)\Big|.\notag
\end{align}
Therefore, we have
\begin{align}
    \max_{(s,a) \in \cS \times \cA}\Big|\qvalue^{(U)}(s,a) - \qvalue^{(U-1)}(s,a)\Big| &\leq \gamma^{U-1}\max_{(s,a) \in \cS \times \cA}\Big|\qvalue^{(1)}(s,a) - \qvalue^{(0)}(s,a)\Big|\notag \\
    & = \gamma^{U-1}\max_{(s,a) \in \cS \times \cA}\bigg|\reward(s,a) + \frac{\gamma}{1-\gamma} - \frac{1}{1-\gamma}\bigg|\notag\\
    & \leq 2\gamma^{U-1},\notag
\end{align}
where the last inequality holds due to the fact that $0\leq \reward(s,a) \leq 1$ for any $(s,a)$. Therefore we prove \eqref{eq:V-diff_0}. To prove the original statement, we have 
\begin{align}
    \qvalue^{(U)}(s_t, a_t) 
    &= \reward(s_t, a_t) + \gamma \max_{\btheta\in\cC_k\cap\cB}\big\la \btheta, \bphi_{\vvalue^{(U-1)}}(s_t, a_t)\big\ra\label{eq:V-diff_4.5} \\
    & = \reward(s_t, a_t) + \gamma \big\la \check\btheta, \bphi_{\vvalue^{(U-1)}}(s_t, a_t)\big\ra\notag \\
    & = \reward(s_t, a_t) + \gamma \check\PP\vvalue^{(U-1)}(s_t, a_t)\notag \\
    & = \reward(s_t, a_t) + \gamma \check\PP\vvalue^{(U)}(s_t, a_t) + \gamma \check\PP[\vvalue^{(U-1)} -\vvalue^{(U)} ](s_t, a_t)\notag \\
    & \leq \reward(s_t, a_t) + \gamma \check\PP\vvalue^{(U)}(s_t, a_t) + \gamma \max_{(s,a) \in \cS \times \cA}\Big|\qvalue^{(U)}(s,a) - \qvalue^{(U-1)}(s,a)\Big| \notag \\
    & \leq \reward(s_t, a_t) + \gamma \check\PP\vvalue^{(U)}(s_t, a_t) + 2\gamma^{U}\notag \\
    & = \reward(s_t, a_t) +  \gamma \big\la \check\btheta, \bphi_{\vvalue^{(U)}}(s_t, a_t)\big\ra + (1-\gamma)/T
    \label{eq:V-diff_5},
\end{align}
where $\check\btheta$ is the $\btheta$ which attains the maximum of \eqref{eq:V-diff_4.5}, $\check\PP(s'|s_t, a_t) = \la \check\btheta, \bphi(s'|s_t, a_t)\ra$. The first inequality holds due to the fact that $|\check \PP f(s_t, a_t)| \leq \max_{s' \in \cS} |f(s')|$ and $\max_s|\vvalue^{(U-1)}(s) - \vvalue^{(U)}(s)| \leq \max_{s,a} |\qvalue^{(U-1)}(s,a) - \qvalue^{(U)}(s,a)|$, the second inequality holds due to \eqref{eq:V-diff_0}. Taking $\btheta_t = \check\btheta$, our conclusion holds. 
\end{proof}

\subsection{Proof of Lemma \ref{lemma:boundk}}
\begin{proof}[Proof of Lemma \ref{lemma:boundk}]
For simplicity, we denote $K = K(T)$. 
Note that $\det(\bSigma_1) = \lambda^d$. We further have
\begin{align}
    \|\bSigma_{T+1}\|_2 &= \bigg\|\lambda\Ib+\sum_{k=0}^{K-1}\sum_{t=t_k}^{t_{k+1}-1}\bphi_{\vvalue_k}(s_t, a_t)\bphi_{\vvalue_k}(s_t, a_t)^\top\bigg\|_2 \notag \\
    & \leq \lambda + \sum_{k=0}^{K-1}\sum_{t=t_k}^{t_{k+1}-1}\big\|\bphi_{\vvalue_k}(s_t, a_t)\big\|_2^2\notag \\
    & \leq \lambda + \frac{Td}{(1-\gamma)^2}, \label{eq:boundk_0}
\end{align}
where the first inequality holds due to the triangle inequality, the second inequality holds due to the fact $\vvalue_k \leq 1/(1-\gamma)$ from Lemma \ref{lemma:UCB} and Definition \ref{assumption-linear}. \eqref{eq:boundk_0} suggests that $\det(\bSigma_{T+1}) \leq (\lambda+Td/(1-\gamma)^2)^d$. Therefore, we have
\begin{align}
    \bigg(\lambda + \frac{Td}{(1-\gamma)^2}\bigg)^d \geq \det(\bSigma_{T+1}) \geq \det(\bSigma_{t_{K-1}})\geq 2^{K-1}\det(\bSigma_{t_{0}}) = 2^{K-1}\lambda^d,  \label{eq:boundk_1}
\end{align}
where the second inequality holds since $\bSigma_T\succeq\bSigma_{t_{K-1}} $, the third inequality holds due to the fact that $\det(\bSigma_{t_{k}}) \geq 2\det(\bSigma_{t_{k-1}})$ by the update rule in Algorithm \ref{algorithm}. \eqref{eq:boundk_1} suggests 
\begin{align}
    K \leq 2d\log \frac{\lambda +Td}{\lambda(1-\gamma)^2}.\notag
\end{align}
\end{proof}

\section{Proof of lemmas in Section \ref{sec:proof:thm1}}

\subsection{Proof of Lemma \ref{lemma:regrettrans}}
\begin{proof}[Proof of Lemma \ref{lemma:regrettrans}]
We have the following equations due to the expectation. 
\begin{align}
    \EE_{\btheta}\text{Regret}(\btheta) &=\EE_{\btheta}\bigg[ \sum_{t=1}^T\vvalue^*(s_t) - \sum_{t=1}^T\sum_{t'=0}^\infty \gamma^{t'}\reward(s_{t+t'}, a_{t+t'})\bigg]\notag \\
    &=\EE_{\btheta}\bigg[ \sum_{t=1}^T\vvalue^*(s_t) - \sum_{t=1}^\infty\sum_{t'=\max\{0, t-T\}}^{t-1} \gamma^{t'}\reward(s_{t}, a_{t})\bigg]\notag \\
    & = \EE_{\btheta}\bigg[ \sum_{t=1}^T\vvalue^*(s_t) - \underbrace{\sum_{t = 1}^T \reward(s_{t}, a_{t}) \sum_{t' = 0}^{t-1} \gamma^{t'}}_{I_1} - \underbrace{\sum_{t = T+1}^\infty \reward(s_{t}, a_{t}) \sum_{t' =t-T }^{t-1} \gamma^{t'}}_{I_2}\bigg].\label{eq:regrettrans_000}
\end{align}
For $I_1$, we have
\begin{align}
    I_1 \leq \sum_{t = 1}^T \reward(s_{t}, a_{t}) \sum_{t' = 0}^{\infty} \gamma^{t'} = \sum_{t = 1}^T \reward(s_{t}, a_{t})/(1-\gamma),\label{eq:regrettrans_00}
\end{align}
where the first inequality holds since $t-1 \leq \infty$. 

For $I_2$, we have
\begin{align}
    I_2 \leq \sum_{t = T+1}^\infty 1\cdot \sum_{t' =t-T }^{t-1} \gamma^{t'} \leq \sum_{t = T+1}^\infty 1\cdot \sum_{t' =t-T }^{\infty} \gamma^{t'} = \sum_{t = T+1}^\infty 1\cdot \frac{\gamma^{t-T}}{1-\gamma} = \frac{\gamma}{(1-\gamma)^2},\label{eq:regrettrans_01}
\end{align}
where the first inequality holds since $\reward(s_t, a_t) \leq 1$, the second inequality holds since $t-1 \leq \infty$. 
Substituting \eqref{eq:regrettrans_00} and \eqref{eq:regrettrans_01} into \eqref{eq:regrettrans_000}, we have
\begin{align}
    \EE_{\btheta}\text{Regret}(\btheta)
    & \geq \EE_{\btheta}\bigg[ \sum_{t=1}^T\vvalue^*(s_t) - \frac{1}{1-\gamma}\sum_{t' = 1}^T \reward(s_{t'}, a_{t'}) - \frac{\gamma}{(1-\gamma)^2}\bigg].\notag
\end{align}

\end{proof}

\subsection{Proof of Lemma \ref{lemma:n1} }
\begin{proof}[Proof of Lemma \ref{lemma:n1}]
We have
\begin{align}
    \EE_{\btheta}N_1
    &= \sum_{t=2}^T \cP_{\btheta} (s_t = \state_1)\notag \\
    & = \underbrace{\sum_{t=2}^T \cP_{\btheta} (s_t = \state_1|s_{t-1} = \state_1)\cP_{\btheta}(s_{t-1} = \state_1)}_{I_1} + \underbrace{\sum_{t=2}^T  \cP_{\btheta} (s_t = \state_1,s_{t-1} = \state_0)}_{I_2}.\label{eq:1111}
\end{align}
For $I_1$, since $\cP_{\btheta} (s_t = \state_1|s_{t-1} = \state_1) = 1-\delta$ no matter which action is taken, thus we have
\begin{align}
    I_1 = (1-\delta) \sum_{t=2}^T\cP_{\btheta}(s_{t-1} = \state_1) = (1-\delta) \EE_{\btheta}N_1 - (1-\delta)\cP_{\btheta}(s_T = \state_1). \label{eq:2}
\end{align}
Next we bound $I_2$. We can further decompose $I_2$ as follows:
\begin{align}
    I_2 &= \sum_{t=2}^T \sum_{\ab}\cP_{\btheta}(s_t = \state_1|s_{t-1} = \state_0, a_{t-1} = \ab) \cP_{\btheta}(s_{t-1} = \state_0, a_{t-1} = \ab)\notag \\
    & = \sum_{t=2}^T \sum_{\ab}(\delta + \la \ab, \btheta\ra) \cP_{\btheta}(s_{t-1} = \state_0, a_{t-1} = \ab)\notag\\
    & = \sum_{\ab} (\delta + \la\ab, \btheta \ra)\Big[ \EE_{\btheta} N_0^{\ab} - \cP_{\btheta}(s_T = \state_0, a_T = \ab)\Big]. \label{eq:3}
\end{align}
Substituting \eqref{eq:2} and \eqref{eq:3} into \eqref{eq:1111} and rearranging it, we have
\begin{align}
    \EE_{\btheta}N_1 
    &= \sum_{\ab} (1+ \la \ab, \btheta\ra/\delta) \EE_{\btheta} N_0^{\ab}  \notag \\
    &\qquad- \underbrace{\bigg[\frac{1-\delta}{\delta}\cP_{\btheta}(s_T = \state_1) + \sum_{\ab}(1+ \la \ab, \btheta\ra/\delta)\cP_{\btheta}(s_T = \state_0, a_T = \ab)\bigg]}_{\psi_{\btheta}}\notag \\
    & =  
    \EE_{\btheta}N_0 + \delta^{-1}\sum_{\ab}\la \ab, \btheta\ra \EE_{\btheta}N_0^{\ab} - \psi_{\btheta},\label{eq:4}
\end{align}
where $\Psi_{\btheta}$ is non-negative because $\la \ab, \btheta\ra \geq -\Delta \geq -\delta$. 
\eqref{eq:4} immediately implies that
\begin{align}
    \EE_{\btheta}N_1 \leq T/2 + \delta^{-1}\sum_{\ab}\la \ab, \btheta\ra \EE_{\btheta}N_0^{\ab}/2.
\end{align}
We now bound $\EE_{\btheta}N_0$. By \eqref{eq:4}, we have
\begin{align}
    \EE_{\btheta}N_1 &= \EE_{\btheta}N_0 + \delta^{-1}\sum_{\ab}\la \ab, \btheta\ra \EE_{\btheta}N_0^{\ab} - \psi_{\btheta}\notag \\
    & \geq \EE_{\btheta}N_0 - \frac{\Delta}{\delta}\EE_{\btheta}N_0 - \frac{1-\delta}{\delta}\cP_{\btheta}(s_T = \state_1) - \bigg[1 + \frac{\Delta}{\delta}\bigg]\cP_{\btheta}(s_T = \state_0)\notag \\
    & = (1-\Delta/\delta)\EE_{\btheta}N_0 - (1-\delta)/\delta + \frac{1-\Delta}{\delta}\cP_{\btheta}(s_T = \state_0)\notag \\
    & \geq (1-\Delta/\delta)\EE_{\btheta}N_0 - (1-\delta)/\delta\label{eq:5},
\end{align}
where the first equality holds due to \eqref{eq:4}, the first inequality holds due to the fact that $\la \ab, \btheta\ra \leq \Delta$, the last inequality holds since $\cP_{\btheta}(s_T  = \state_0)>0$. \eqref{eq:5} suggests that
\begin{align}
    \EE_{\btheta}N_0 \leq \frac{T+(1-\delta)/\delta}{2 - \Delta/\delta} \leq \frac{4}{5}T, \notag
\end{align}
where the last inequality holds due to the fact that $2\Delta \leq \delta$ and $(1-\delta)/\delta < T/5$.  
\end{proof}

\subsection{Proof of Lemma \ref{lemma:kl}}
We need the following lemma:
\begin{lemma}[Lemma 20 in \citet{jaksch2010near}]\label{lemma:basicin}
Suppose $0 \leq \delta' \leq 1/2$ and $\epsilon' \leq 1-2\delta'$, then
\begin{align}
    \delta'\log\frac{\delta'}{\delta'+\epsilon'} + (1-\delta')\log\frac{(1-\delta')}{1-\delta'-\epsilon'} \leq \frac{2(\epsilon')^2}{\delta'}. \notag
\end{align}
\end{lemma}
\begin{proof}[Proof of Lemma \ref{lemma:kl}]
Let $\sbb_t$ be $\{s_1,\dots,s_t\}$. By the Markovian property of MDPs, we can first decompose the KL divergence as follows:
\begin{align}
    \text{KL}(\cP_{\btheta'}\| \cP_{\btheta}) = \sum_{t=1}^{T-1}\text{KL}\Big[\cP_{\btheta'}(s_{t+1}|\sbb_{t})\Big\| \cP_{\btheta}(s_{t+1}|\sbb_{t})\Big],\notag
\end{align}
where the KL divergence between $\cP_{\btheta'}(s_{t+1}|\sbb_{t}), \cP_{\btheta}(s_{t+1}|\sbb_{t})$ is defined as follows:
\begin{align}
    \text{KL}\Big[\cP_{\btheta'}(s_{t+1}|\sbb_t)\Big\| \cP_{\btheta}(s_{t+1}|\sbb_t)\Big] = \sum_{\sbb_{t+1} \in \cS^{t+1}}\cP_{\btheta'}(\sbb_{t+1})\log\frac{\cP_{\btheta'}(s_{t+1}|\sbb_t)}{\cP_{\btheta}(s_{t+1}|\sbb_t)}.\notag
\end{align}
Now we further bound the above terms as follows:
\begin{align}
     &\sum_{\sbb_{t+1} \in \cS^{t+1}}\cP_{\btheta'}(\sbb_{t+1})\log\frac{\cP_{\btheta'}(s_{t+1}|\sbb_t)}{\cP_{\btheta}(s_{t+1}|\sbb_t)} \notag \\
     &\qquad= \sum_{\sbb_t \in \cS^{t}}\cP_{\btheta'}(\sbb_t)\sum_{\state \in \cS} \cP_{\btheta'}(s_{t+1} = \state|\sbb_t)\log\frac{\cP_{\btheta'}(s_{t+1} =\state|\sbb_t)}{\cP_{\btheta}(s_{t+1} = \state|\sbb_t)}\notag \\
     & \qquad= \sum_{\sbb_{t-1} \in \cS^{t-1}}\cP_{\btheta'}(\sbb_{t-1})\sum_{\state' \in \cS, \ab \in \cA}\cP_{\btheta'}(s_t = \state', a_t = \ab|\sbb_{t-1})\notag \\
     &\qquad\qquad\cdot \sum_{\state \in \cS} \cP_{\btheta'}(s_{t+1} = \state|\sbb_{t-1}, s_t = \state', a_t = \ab)\underbrace{\log\frac{\cP_{\btheta'}(s_{t+1} =\state|\sbb_{t-1}, s_t = \state', a_t = \ab )}{\cP_{\btheta}(s_{t+1} =\state|\sbb_{t-1}, s_t = \state', a_t = \ab )}}_{I_1}\notag ,
\end{align}
where $\cS = \{\state_0,\state_1\}$.
When $\state' = \state_1$, by the definition of the hard MDP constructed in Section \ref{subsec:lower bound}, we have $\cP_{\btheta'}(s_{t+1} =\state|\sbb_{t-1}, s_t = \state', a_t = \ab ) = \cP_{\btheta}(s_{t+1} =\state|\sbb_{t-1}, s_t = \state', a_t = \ab )$ for all $\btheta', \btheta$ since the transition probability at $\state_1$ is irrelevant to $\btheta$ due to the MDP we choose. This implies when $\state' = \state_1$, $I_1 = 0$. Therefore, 
\begin{align}
    &\sum_{\sbb_{t+1} \in \cS^{t+1}}\cP_{\btheta'}(\sbb_{t+1})\log\frac{\cP_{\btheta'}(s_{t+1}|\sbb_t)}{\cP_{\btheta}(s_{t+1}|\sbb_t)}\notag \\
    & \qquad= \sum_{\sbb_{t-1} \in \cS^{t-1}}\cP_{\btheta'}(\sbb_{t-1})\sum_{\ab}\cP_{\btheta'}(s_t = \state_0, a_t = \ab|\sbb_{t-1})\notag \\
     &\qquad\qquad\cdot \sum_{\state \in \cS} \cP_{\btheta'}(s_{t+1} = \state|\sbb_{t-1}, s_t = \state_0, a_t = \ab)\log\frac{\cP_{\btheta'}(s_{t+1} =s|\sbb_{t-1}, s_t = \state_0, a_t = \ab )}{\cP_{\btheta}(s_{t+1} =s|\sbb_{t-1}, s_t = \state_0, a_t = \ab )}\notag \\
     & \qquad= \sum_{\ab}\cP_{\btheta'}(s_t = \state_0, a_t = \ab)\notag \\
     &\qquad\qquad \cdot\underbrace{\sum_{\state \in \cS} \cP_{\btheta'}(s_{t+1} = s| s_t = \state_0, a_t = \ab)\log\frac{\cP_{\btheta'}(s_{t+1} =\state| s_t = \state_0, a_t = \ab )}{\cP_{\btheta}(s_{t+1} =\state| s_t = \state_0, a_t = \ab )}}_{I_2}\label{eq:kl:1}.
\end{align}
To bound $I_2$, due to the structure of the MDP, we know that $s_{t+1}$ follows the Bernoulli distribution over $\state_0$ and $\state_1$ with probability $1-\delta - \la \ab, \btheta'\ra$ and $\delta + \la \ab, \btheta'\ra$, then we have
\begin{align}
    I_2 &= (1-\la \btheta', \ab\ra - \delta)\log\frac{1-\la \btheta', \ab\ra - \delta}{1-\la \btheta, \ab\ra - \delta} + (\la \btheta', \ab\ra + \delta)\log\frac{\la \btheta', \ab\ra + \delta}{\la \btheta, \ab\ra + \delta}\leq \frac{2\la \btheta' - \btheta,\ab\ra^2}{\la \btheta', \ab\ra+\delta},\label{eq:kl:1.4}
\end{align}
where the inequality holds due to Lemma \ref{lemma:basicin} with $\delta' = \la \btheta', \ab\ra+\delta$ and $\epsilon' = \la \btheta - \btheta', \ab\ra$. Specifically, it can be verified that
\begin{align}
    \delta' = \la \btheta', \ab\ra+\delta \leq \Delta+\delta \leq 1/2,\label{eq:xxx1}
\end{align}
where the first inequality holds due to the definition of $\btheta'$, the second inequality holds since $\Delta<\delta/2 \leq 1/6$. It can also be verified that
\begin{align}
    \epsilon' = \la \btheta - \btheta',\ab\ra \leq 2\Delta \leq 1-2(\Delta+\delta) \leq 1-2\delta',\label{eq:xxx2}
\end{align}
where the first inequality holds due to the definition of $\btheta', \btheta$, the second inequality holds since $\Delta<\delta/4 \leq 1/12$, and the last inequality holds since $\delta' = \la \btheta', \ab\ra+\delta \leq \Delta+\delta$ due to the definition of $\btheta'$. \eqref{eq:xxx1} together with \eqref{eq:xxx2} show that we can indeed apply Lemma \ref{lemma:basicin} to the last step of \eqref{eq:kl:1.4}. $I_2$ can be further bounded as follows:
\begin{align}
        I_2&\leq \frac{4\la \btheta' - \btheta,\ab\ra^2}{\delta} = \frac{16\Delta^2}{(d-1)^2\delta},\label{eq:kl:2}
\end{align}
where the inequality holds due to \eqref{eq:kl:1.4} and the fact that $\delta+\la \btheta', \ab\ra \geq \delta - \Delta \geq \delta/2$.
Substituting \eqref{eq:kl:2} into \eqref{eq:kl:1}, taking summation from $t = 1$ to $T-1$, we have
\begin{align}
        \text{KL}(\cP_{\btheta'}\| \cP_{\btheta})& = \sum_{t=1}^{T-1}\sum_{\sbb_{t+1} \in \cS^{t+1}}\cP_{\btheta'}(\sbb_{t+1})\log\frac{\cP_{\btheta'}(s_{t+1}|\sbb_t)}{\cP_{\btheta}(s_{t+1}|\sbb_t)}\notag \\
        & \leq 
        \frac{16\Delta^2}{(d-1)^2\delta}\sum_{t=1}^{T-1}\sum_{\ab}\cP_{\btheta'}(s_t = \state_0, a_t = \ab)\notag \\
        & =         \frac{16\Delta^2}{(d-1)^2\delta}\sum_{t=1}^{T-1}\cP_{\btheta'}(s_t = \state_0)\notag \\
        & \leq \frac{16\Delta^2}{(d-1)^2\delta}\EE_{\btheta'}N_0,\notag
\end{align}
where the last inequality holds due to the definition of $N_0$. 
\end{proof}

\bibliographystyle{ims}
\bibliography{reference}
\end{document}